\documentclass[runningheads]{llncs}

\usepackage{graphicx}

\usepackage{amsmath}
\usepackage{comment}
\usepackage{amsmath,amssymb}
\usepackage{color}
\usepackage{url}
\usepackage{hyperref}

\usepackage[dvipsnames]{xcolor}

\usepackage[utf8]{inputenc} 
\usepackage[T1]{fontenc}    
\usepackage{hyperref}       
\usepackage{url}            
\usepackage{booktabs}       
\usepackage{amsfonts}       
\usepackage{nicefrac}       
\usepackage{microtype}      
\usepackage{xcolor}         

\usepackage{caption}
\usepackage{pgfplots}
\usepackage{pgfplotstable}
\usepackage{subcaption}
\usepackage{tikz}
\usepackage{mathtools}
\usepackage{placeins}
\usepackage{capt-of}
\usepackage{xcolor,colortbl}
\usepackage{float}
\usepackage{stfloats}
\usepackage{cleveref}

\usepackage{xr}

\newcommand{\eg}{\emph{e.g}.\ } 
\newcommand{\ie}{\emph{i.e}.\ }

\newcommand{\etal}{\emph{et al}.}

\usepgfplotslibrary{fillbetween}

\pgfplotsset{compat=1.17}

\newif\ifreview
\reviewfalse

\ifreview
	\usepackage{lineno}

	\linenumbers
\fi

\begin{document}


\def\SubNumber{45}

\def\GCPRTrack{Main Track}

\title{Self-Masking Networks for Unsupervised Adaptation}

\ifreview
\else

	\author{Alfonso Taboada Warmerdam\inst{1}\orcidID{0000-0002-1767-367X} \and
	Mathilde Caron\inst{2}\orcidID{0000-0001-6594-6698} \and
	Yuki M. Asano\inst{1}\orcidID{0000-0002-8533-4020}}
	
	\authorrunning{Warmerdam et al.}
	
	\institute{University of Amsterdam, Amsterdam \and Google Research, Grenoble 
 }
\fi

\maketitle              

\begin{abstract}
        With the advent of billion-parameter foundation models, efficient fine-tuning has become increasingly important for the adaptation of models to downstream tasks. However, especially in computer vision, it can be hard to achieve good performance when access to quality labeled data is lacking. In this work, we propose a method adapting pretrained generalist models in a self-supervised manner by learning binary masks. 
        These \textbf{s}elf-supervised \textbf{m}asking \textbf{n}etworks (SMNs) are up to 79x more efficient to store and significantly improve performance on label-efficient downstream tasks. 
        We validate the usefulness of learning binary masks as a fine-tuning method on  8 datasets and 3 model architectures, and we demonstrate the effectiveness of SMNs in 3 label-efficient settings.

\end{abstract}

\setlength{\abovedisplayskip}{0pt}
\setlength{\belowdisplayskip}{7pt}
\setlength{\abovedisplayshortskip}{0pt}
\setlength{\belowdisplayshortskip}{7pt}

    \begin{table*}[bp]
    \centering
    \begin{tabular}{c c}
        \includegraphics[width=0.478\linewidth]{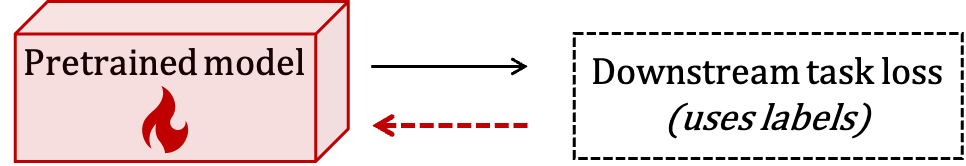} & 
        \includegraphics[width=0.478\linewidth]{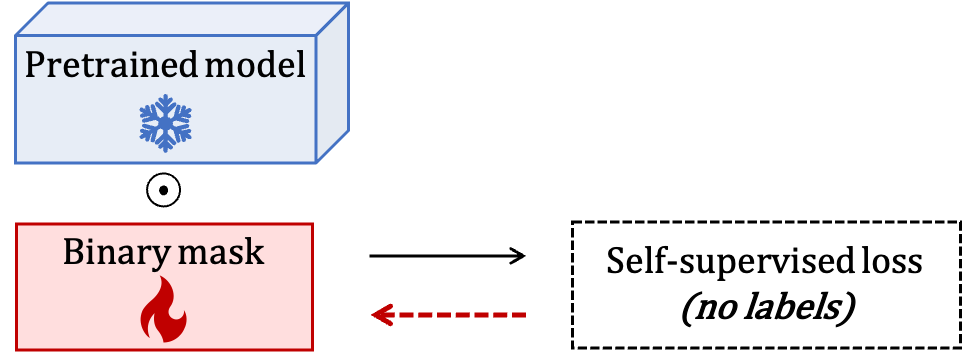} \\
        \textbf{\underline{Standard full-finetuning}} &
    \textbf{\underline{Self-supervised self-masking}} (Ours) \\
    \end{tabular}
    \vspace{+0.2cm}
    \captionof{figure}{\textbf{Conceptual comparison between two adaptation mechanisms: } Standard full-finetuning \textit{versus} self-supervised self-masking (Ours).
        \label{fig:method}}
    \end{table*}

    \section{Introduction}
    

Recent advancements in large-scale pretraining of visual encoders, such as CLIP~\cite{radford2021learning}, DINO~\cite{caronEmergingPropertiesSelfSupervised2021} or MAE \cite{he2022masked}, have shown remarkable and generalizable performances across a variety of computer vision tasks. 
Conventionally, these networks are fine-tuned for specific downstream tasks by adjusting their weights through gradient descent, either by training an additional layer on top of the pretrained network or by fine-tuning the entire network. One limitation of full-fine-tuning is that it necessitates a copy of all the fine-tuned weights to be stored for each downstream task. 
This leads to significant memory requirements, especially as vision models are reaching parameter counts of >22 Billion~\cite{dehghani2023scaling}. 
In contrast, simpler methods like linear probing are often limited in how well they can adapt the base model, as they only have access to the frozen deep representations.

Previous attempts at solving this challenge include adding light-weight learnable adapters~\cite{rebuffi2017learning}, prompt learning~\cite{Loedeman2022prompt,bahngExploringVisualPrompts2022,jia2022visual}, where additional inputs are learned, side-tuning~\cite{zhang2020side} and model soups~\cite{wortsman2022model}, wherein another network is trained and then `fused' with the frozen model, and head-to-toe adaptation~\cite{evci2022head2toe}, which uses intermediate features from all layers of the pretrained network to train a classification head. However, these attempts are usually only evaluated in a supervised setting, while in practice the availability of labels for downstream task adaptation can be lacking.


In this paper, we explore the potential of finding subnetworks within trained networks in a self-supervised manner in order to both {reduce the memory needs of fine-tuned models} and {adapt effectively to downstream tasks in label-sparse situations}. 
The memory requirements compared to standard fine-tuning techniques are reduced because masks, with their binary weights, are more compact than full copies of the network weights. Hence, it is possible to adapt a model to thousands of downstream tasks by storing thousands of cheap binary weights and only one copy of the original model weights. To further explore the potential of this method, we also explore the possibility of chaining multiple adapted models together to yield `model cascades': Here, we train multiple binary masks on different, coherent subsets of the downstream dataset in a self-supervised manner in order to increase downstream accuracy with a negligible increase in model size. This is done in a fully self-supervised manner.

    \section{Related work}


\paragraph{Frozen Network Adaptation.}\label{subsec:network-adapt}



Recent work has shown promising results in adapting pretrained models without modifying the core network weights, which would in theory preserve the original models' performance while enabling it to cater to novel tasks. One such technique is the application of lightweight feature adapters \cite{gao2021clip,zhang2021tip}, which introduce trainable parts into the network while keeping the majority of the model frozen.  Other methods include fine-tuning only the bias parameters of pretrained models \cite{cai2020tinytl,ben-zaken-etal-2022-bitfit}, learning low-rank adaptations~\cite{hulora}, or learning additional inputs to one or multiple layers in a pretrained vision transformer, such as Visual Prompting \cite{bahng2022visual,elsayed2018adversarial,kloberdanz2021improved}, which only learns additional inputs, or \cite{jia2022visual}, which also fine-tunes a linear classifier on top of the model.

\paragraph{Masking Neural Networks.}\label{subsec:pruning-and-neural-network-architectures}

Mallya~\etal~\cite{mallya2018piggyback} were the first to use the pass-through trick, a method that learns subnetwork masks directly through gradient descent. They applied this method to pretrained models to achieve domain adaptation, the new method outperformed their previous work \cite{mallyaPackNetAddingMultiple2018} and reached performance very close to conventional fine-tuning. Ramanujan~\etal~\cite{ramanujanWhatHiddenRandomly2020} found that a similar technique is also able to find well-performing domain-specific subnetworks on random, untrained networks, with implications regarding the Lottery Ticket Hypothesis \cite{frankleLotteryTicketHypothesis2019,malachProvingLotteryTicket}. However, they found that the base network had to be significantly wider in order to perform as well as standard, weight-modifying gradient descent. 
In follow-up work, Wortsman~\etal~\cite{wortsmanSupermasksSuperposition} expand on this approach to teach a random, untrained network thousands of tasks.
A similar technique, in combination with training the actual neural network weights has also been applied as a pruning technique for CNNs and vision transformers~\cite{sanhMovementPruningAdaptive,lagunasBlockPruningFaster2021} as well as for domain generalization \cite{chattopadhyay2020learning}.

Our masking method is functionally equivalent to that of \cite{mallya2018piggyback}, however, we remove redundant hyperparameters (see \cref{sec:hyfr}) and we simplify the scaling necessary to keep the variance constant when neural network weights are masked. We do this by scaling the network weights during forward propagation rather than the gradients (see Eq. \ref{eq:inference}). This is similar to Ramanujan~\etal~\cite{ramanujanWhatHiddenRandomly2020}'s method, but since they fix the percentage of weights that will be masked beforehand, they can instead scale the weights once during initialization.

Other fields where this technique has been applied include Neural Architecture Search~\cite{wortsmanDiscoveringNeuralWirings2019}, where the pass-through trick  is used to determine where to add or remove a connection, and the design of novel neural network architectures such as neural networks with binary weights (-1 or 1) and ternary weights (-1, 0 or 1)~\cite{courbariaux2015binaryconnect,courbariaux2016binarized,liTernaryWeightNetworks2022,zhu2016trained}.
These works have demonstrated the effectiveness of the pass-through trick for learning certain kinds of discrete components through gradient descent. With regard to sparsity in general, \cite{hoefler2021sparsity} have provided a good overview of the different use-cases and methods for sparsity in neural networks.

\paragraph{Self-supervised Learning on Restricted Domains.}\label{subsec:self-supervised-learning}

Recently, self-supervised approaches have gained traction as a way to improve the performance of computer vision models both in general and in label sparse situations~\cite{caronEmergingPropertiesSelfSupervised2021,caronUnsupervisedLearningVisual2021,he2020momentum}.
These approaches enable a model backbone to be trained without the use of labels, where the focus has mostly been on learning good foundation models. However, some work has also shown that self-supervised learning can be used to improve performance on downstream tasks in computer vision, either through self-supervised knowledge distillation \cite{fangSEEDSelfsupervisedDistillation2021,ericsson2021well}, or full-network adaptation \cite{reedSelfSupervisedPretrainingImproves2022}, which has also been done in natural language processing \cite{gururanganDonStopPretraining2020,howard2018universal}. We build on the idea of self-supervised adaptation and apply it to mask-learning with the simple yet effective self-supervised SwAV loss \cite{caronUnsupervisedLearningVisual2021}.

     
    \section{Method}

Our goal is to adapt pretrained networks in a storage-efficient manner without labels. 
In a nutshell, our method learns a binary mask for a pretrained network with a self-supervised loss.
In this section, we will recall the background on learning masks for neural networks followed by our contributions of provably removing unnecessary hyperparameters and show how this method can be used to learn multiple expert models for a given domain to improve adaptation performance.
\subsection{Background: Network Masking}
    \label{section:masking}
    Subnetworks are represented by a binary mask $M$ which indicates the active weights and the ones that are zeroed out. To learn an appropriate mask, each weight is assigned a corresponding score $s$, initialized to a value higher than a threshold $\mu$~\cite{ramanujanWhatHiddenRandomly2020}. 
    At inference, a weight $\theta_i$ is then considered `active' if its score is larger than $\mu$, otherwise, the mask is set to zero and the weight is unused:

    \begin{equation}
        \theta_i = \frac{\theta_i}{\alpha} \cdot M_{\theta_i}  \label{eq:inference}
    \end{equation},
    
    where $M_{\theta_i} = \mathbb{I}[S_{\theta_i} > \mu]$, and $\alpha= \sqrt{\frac1{N}\sum_{i=1}^{N}{\mathbb{I}[S_{\theta_i} > \mu]}}$ is a scaling term to keep the variance of a weight matrix with $N$ parameters constant when weights are deactivated~\cite{ramanujanWhatHiddenRandomly2020}.



    \paragraph{Pass-through trick for training.}
    Since the gradient of the score is lost when a weight is deactivated, the score cannot be learned directly. Instead, the gradient with respect to the mask is used to update the scores:
    %
%
    $S^{t} = S^{t-1} - \lambda \frac{d\mathcal{L}^{t-1}}{dM^{t-1}},$
    where $\mathcal{L}^{t-1}$ denotes the loss at time step ${t-1}$, and $\lambda$ the learning rate.
%
    The key to implementing this in practice is to set the gradient of the score with respect to the loss to be the gradient of the mask with respect to the loss, 
    $\frac{d\mathcal{L}^{t}}{dS^{t}} = \frac{d\mathcal{L}^{t}}{dM^{t}}$.

    Which results in
    \begin{align}
    S^{t} &= S^{t-1} - \lambda \frac{d\mathcal{L}^{t-1}}{dS^{t-1}}, 
    \label{eq:update_scores2}
    \end{align}

    which is the standard Stochastic Gradient Descent (SGD) update equation for parameter $S$. This approach, called the pass-through trick, thus allows standard SGD algorithms to update the score and learn a binary mask.
    
    Intuitively, this method is able to optimize the loss as follows: The update equation makes sure that if increasing the mask would reduce the loss, the score is increased, and vice versa. 
    Through multiple gradient descent iterations, the score then decreases for weights that increase the loss, eventually reaching the threshold and deactivating the weight. Conversely, if a deactivated weight would now reduce the loss when active, the score will increase, and the weight will eventually be reactivated.

\begin{figure*}
    \label{fig:cascade}
    \centering
    \includegraphics[width=0.99\linewidth]{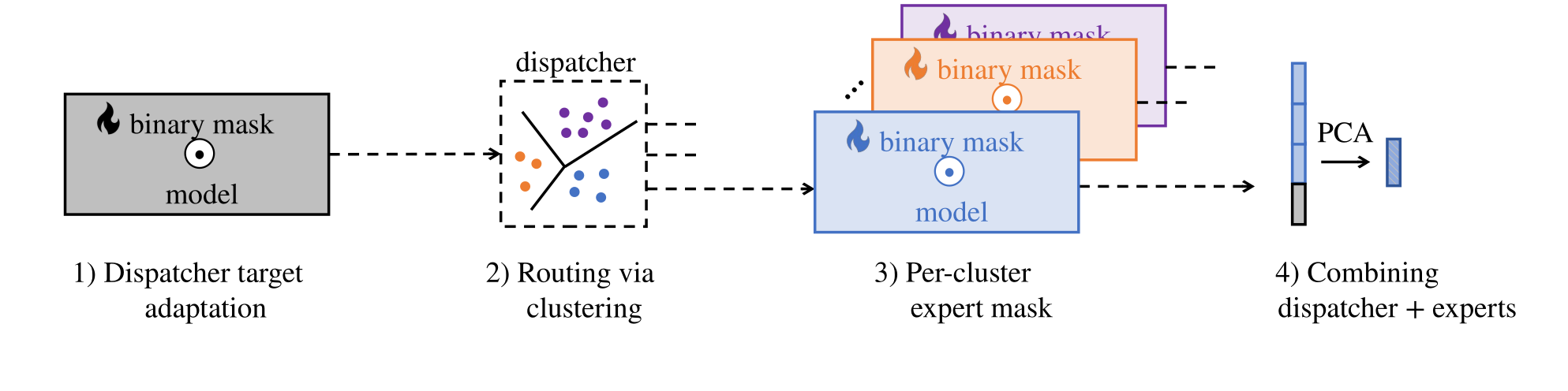}
    \vspace{-0.4cm}
    \caption{\textbf{Cascade models} work by complementing the root model's features with those of expert models which are tailored to specific parts of the training distribution.}
\end{figure*}

\subsection{Hyperparameter-free masking}
    \label{sec:hyfr}

    The existing approach can learn a mask for a network given some loss function, this comes however at the cost of introducing two additional hyperparameters when compared to standard full-fine-tuning: the threshold $\mu$ and the initial score value $S^0$. 
    We show that it is possible to remove \textit{both} hyperparameters:
    
    \begin{theorem}
    Translation invariance of threshold and initialization. Shifting the score initialisation $S^0$ and the threshold $\mu$ by an equal amount does not affect SGD-based training without weight-decay. \label{theo:1}
    \end{theorem}
    
    The proof follows from the idea that if the initial mask remains the same after translating both parameters equally, then the gradient will also be the same and so will the mask in the next training step. 
    
    \begin{proof}
        
        As detailed in Section \ref{section:masking}
        , the mask $M_{\theta_i}^t$ for a given weight $\theta_i$ on the training step $t$ is given by
        
        \begin{align*}
           M_{\theta_i}^t &= \mathbb{I} [S_{\theta_i}^t > \mu]
        \end{align*}. For ease of notation, we omit $\theta_i$ from this point on, i.e. $ M_{\theta_i}^t =  M_{}^t$ and $S_{\theta_i}^t = S_{}^t$.
         Now, given the update equation (Eq. 2)
         , $S_{}^t$ can be formulated as
        
        \begin{align*}
            S_{}^t &= S_{}^0 - \sum_{\hat t=0}^{t-1} \lambda^{\hat t} g_{i}(\mathcal{M}^{\hat t}, \mathcal{B}^{\hat t})
        \end{align*}
        
        where $g_{i}(\mathcal{M}^{\hat t}, \mathcal{B}^{\hat t})$ is the gradient of the the weight $\theta_i$'s score given the mask $\mathcal{M}^{\hat t}$ and mini-batch $\mathcal{B}^{\hat t}$.
        Note that the gradient can be fully determined by these two variables (plus all constant parameter such as the model weights).
        Note also that we assume the computation of the gradient to be deterministic, i.e. the same gradient will be computed for the same input, and the input $\mathcal{B}^{\hat t}$ is only dependent on $\hat t$.
        Combining these two equations we get
        
        \begin{align}
            M_{}^0 &= [ S_{}^0 > k] 
            \\&= [ S_{}^0 + a > k + a] \label{eq:basecase} \\
            M_{}^t &= [ S_{}^0 - \sum_{\hat t=0}^{t-1} \lambda^{\hat t} g_{i}(\mathcal{M}^{\hat t}, \mathcal{B}^{\hat t}) > k] 
            \\&= [ S_{}^0 + a - \sum_{\hat t=0}^{t-1} \lambda^{\hat t} g_{i}(\mathcal{M}^{\hat t}, \mathcal{B}^{\hat t}) > k + a]\label{eq:recur}
        \end{align}
        .
        
        Eq. \ref{eq:basecase} indicates that replacing ${ S}_i^0$ with $ S_{}^0 + a$ and $k$ with $k+a$ will not change $M_{}^0$, since this is true for every individual mask, the full network mask $\mathcal{M}^0$ is also invariant to this change.
        This means that the gradient $g_i(\mathcal{M}^0, \mathcal{B}^0)$ does not change, and consequently, $M_{}^1$ is also invariant to this change, as indicated by Eq. \ref{eq:recur}. The same reasoning can be applied recursively to $M_{}^2$ and so on. Thus, by induction, translating the initial score and threshold by equal amounts will not change any of the network masks during training with SGD w/o weight decay.
    \end{proof}

    This eliminates the threshold parameter $\mu$ and leaves us with only the set of initialization scores $S^0$.
    However, even these can be subsumed into an already existing hyperparameter for training, the learning rate:
    \begin{theorem}
    Learning rate and score initialization equivalence. Scaling the score initialization $S^0$ by a factor $\alpha$ is equivalent to scaling the learning rate $\lambda$ by a factor $\frac{1}{\alpha}$. \label{theo:2}
    \end{theorem}

    \begin{proof}
        
        
        
        The equation for SGD with weight decay is given below:
        
        \begin{align*}
             S_{}^t &=  S_{}^{t-1} - \lambda^t \left( g_i(\mathcal{M}^{t-1}, \mathcal{B}^{t-1}) + \gamma  S_{}^{t-1} \right)
        \end{align*}
        
        Say we scale the learning rates and scores by $\alpha$, and the weight decay by $1/\alpha$.
        I.e. we replace $\lambda^t$ with $\lambda^t \alpha$, $ S_{}^{t-1}$ with $ S_{}^{t-1} \alpha$ and $\gamma$ with $\frac{\gamma}{\alpha}$ for some $\alpha \in \mathbb{R}^+$, then we arrive at:
        
        \begin{align}
            \alpha S_{}^{t-1} - \alpha &\lambda^t \left( g_i(\mathcal{M}^{t-1}, \mathcal{B}^{t-1}) + \frac{\gamma}{\alpha} \alpha S_{}^{t-1} \right) \\&= \alpha \left( S_{}^{t-1} - \lambda^t \left( g_i(\mathcal{M}^{t-1}, \mathcal{B}^{t-1}) + \gamma  S_{}^{t-1} \right)\right) \nonumber \\
            &= \alpha S_{}^t\label{eq:krrrt}
        \end{align}
        
        In other words, the update equation then provides the same updated score, except it is also scaled by $\alpha$, like the input score.
        
        Similarly to the previous proof, the initial masks $M_{}^0 = [ S_{}^0 > 0]$ are invariant to the scale change $M_{}^0 = [\alpha S_{}^0 > 0]$, so the replacement of $ S_{}^0$ with $\alpha S_{}^0$, combined with the other replacements, does not change the gradient $g_i(\mathcal{M}^0, \mathcal{B}^0)$. Combined with eq \ref{eq:krrrt}, this means that the updated parameter after the first SGD step is only different in scale when compared to what it would have been without the scale change ($=\alpha S_{}^t$). Apply this reasoning recursively and it can be seen through induction that the network masks will be the same during training as for the original learning rate, score initialization and weight decay.
        
    \end{proof}

    \paragraph{Practical implications}
    
    Note that, in the interest of keeping the proofs simple, momentum is omitted from the derivations. However, we ran the experiment in Appendix \ref{apx:expi}, which validates these proofs experimentally, with momentum enabled. This indicates that this invariance also holds if momentum is included. Also note that other SGD settings, which we left disabled in our research, may break the proven invariances. The translation invariance for instance \textit{does not} hold if weight decay is enabled.
    
    These two theorems indicate that, if we keep weight decay disabled, the hyperparameter combination $\lambda=50$, $S^0=1.0$, $\mu=0.0$ is equivalent to the drastically different hyperparameter combination $\lambda=100$, $S^0=2.5$, $\mu=0.5$ for example. In theory this means that masking the same model under exactly the same conditions, except choosing a different equivalent hyperparameter combination, will result in exactly the same output masks. For this reason, we run all our masking experiments with $\mu=0$ and $S_0=1$ with no loss of expressiveness, provided that weight decay is disabled (which we also disable). We show in Appendix 
    \ref{apx:expi} that this theoretical result also applies in practice.

\subsection{Label-free adaptation}

    Since~\cref{eq:update_scores2} can work with any loss function $\mathcal{L}$. We can make this approach work on arbitrary data without any annotations by utilizing a self-supervised loss based on optimal-transport clustering~\cite{asano2020self,caronUnsupervisedLearningVisual2021}:
    \begin{align}
      &\ell(x_t, x_s)) = - \sum_{k} \mathbf{\Phi}(x)_s^{(k)} \log \mathbf{p}_t^{(k)}, \quad
      &\text{with} \quad \mathbf{p}_t^{(k)} = \frac{ \exp \left ( \frac{1}{\tau} \mathbf{\Phi}(x_t)^\top \mathbf{c}_k \right ) }{\sum_{k'} \exp \left ( \frac{1}{\tau} \mathbf{\Phi}(x_t)\top \mathbf{c}_{k'} \right ) }
      \label{eq:loss}
    \end{align}
    where $x_t, x_s$ are two augmented views of an image $x$, $c_k$ are a set of entropy-regularised prototypes learned with Sinkhorn-Knopp~\cite{cuturi2013sinkhorn} and, in this case the visual encoder $\Phi$ is parameterised with a set of frozen weights $\theta$ and the learnable masking scores $S$.

\subsection{Model Cascades}
    Besides training separate masks for adapting a single model efficiently to several domains, we next show how this approach can be used to obtain a set of fine-grained adapted models on a single dataset without supervision. 
    To this end, we first adapt a pretrained model in the manner described above on some target dataset $\mathcal{D}$ to obtain the adapted network $\Phi_0$. 
    Next, the deep embeddings of this network are used to obtain a non-parametric routing function $R$ via clustering into $K$ sets. 
    Finally, starting from the adapted network's initialization, each subset $\mathcal{D}_k$ is used to learn an `expert' network.
    This follows from the findings of self-supervised models performing~\cite{ericsson2021well} and training~\cite{tian2021divide} better in-domain.
    At inference, the $R$ (\eg a trained $k$-means module) simply routes based on the proximity to the nearest centroid to a single expert.

    \paragraph{Combining Embeddings.}
    Since the experts are trained independently of each other and from the dispatcher (besides starting off with the same weight scores), each model may produce wildly different embeddings for the same data point. Care must thus be taken regarding how these embeddings are combined. In addition, we require the model cascade to provide image embeddings of the same dimensionality as the original adapted model (the dispatcher). This is necessary in order to be able to do a fair comparison using a linear probe between the embedding quality from just the dispatcher versus the full cascade. Namely, the number of parameters of the supervised evaluation component, \ie the weight matrix of the linear probe is the same for just the dispatcher and the cascade.

    We thus decide to combine the individual models' embeddings with Principal Component Analysis (PCA) applied to a concatenation of the embeddings, ensuring that the output dimensionality is the same as that of the dispatcher alone. We evaluate two different ways to concatenate the embeddings. Which are \textbf{unconditional} and \textbf{conditional} concatenation.

    In the \textbf{unconditional} case, the dispatcher embedding $D(x)$ and all expert embeddings $E_1(x),\mathellipsis, E_K(x)$ are concatenated for each datapoint.
    The concatenated embedding $\overline{e}$ is then given by
    
    \begin{equation}
        \overline{e} = [D(x), E_1(x),...,E_K(x)]
    \end{equation}


    %
    In contrast, for the \textbf{conditional} case, only the dispatcher embedding $D(x)$ and one expert embedding $E_n(x)$ are combined for each datapoint.
    In this case, the expert is chosen by the router $R$ by simply picking the cluster closest to the datapoint. 
    The concatenated embedding $\overline{e}$ is then given by
    
   \begin{equation} 
   \overline{e} = [D(x), ..., E_i(x), ...], \,\text{where } i=R(x)
   \end{equation}

    The embedding is padded to have the same length as the unconditional concatenation, where the embeddings from each expert are placed in different dimensions. Effectively, the concatenated embedding is the same as for the unconditional case, except that the expert embeddings from other clusters are zero'ed out. 
    Consequently, only two forward passes (dispatcher and one expert) are needed to embed a new datapoint, as opposed to $K$.

    \paragraph{Dimensionality Reduction.}
    As these two approaches increase the dimensionality of the embeddings by a factor 6, we reduce dimensionality back to the original size ($F=2048$) for a fair comparison to the original model.
    To this end, we first center the features such that $e_i' = \overline{e_i} - c_i$, where $c_i$ is the mean of the $i$'th feature over the (concatenated) training set.
    The final cascade embedding is then given by $e^* = \mathrm{diag}(1/S)V^T\overline{e}$, where $V \in \mathbb{R}^{K\cdot F \times F}$ is the matrix of eigenvectors with the $F$ largest eigenvalues of the covariance matrix of the centered training set embeddings and $S$ the vector of eigenvalues. We divide by the eigenvalues to make sure all features have approximately equal variance, which is expected by the linear probe algorithm applied afterwards in some experiments.
     
    \section{Experiments}

\begin{table*}[
]

     \caption{
      \textbf{Masking is a viable strategy for downstream task adaptation.} 
      We report top-1 accuracy for nine image classification benchmarks with four different methods for downstream task transfer.
      For this experiment, we use supervised masks.
      We report the memory storage (in MegaBits, uncompressed) required to store the weights for a downstream task adaptation using the given method in the \textit{Size} column. This excludes having to store the original model weights as these only need to be stored once.\\
}
    \label{tab:masking_sup}
\centering
\small
  \setlength{\tabcolsep}{1.8pt}
     \begin{tabular}{@{}l c ccccccccc}
     \toprule
     Method & Size & \textsc{cifar10} & \textsc{cifar100} & \textsc{dtd} & \textsc{eurosat} & \textsc{flowers} & \textsc{pets} & \textsc{sun397} & \textsc{ucf101} \\
     \midrule
\multicolumn{6}{l}{\textit{Pretrained model: $\text{ResNet-18}_{\text{Supervised}}$}} \\
     $k$-NN & n/a & {0.826}	& 0.589	& 0.587	& 0.896	& 0.677	& 0.877	& {0.465}	& 0.596\\
     FFT & 368 & \textbf{0.950} & 0.759 & \textbf{0.692} & \textbf{0.969} & \textbf{0.963} &  \textbf{0.889} & \textbf{0.534} & \textbf{0.689} \\
    Mask & 12 & 0.949 & \textbf{0.760} & 0.660 & \textbf{0.969} & 0.955  & 0.852 & 0.479 & 0.659 \\
     \midrule
\multicolumn{6}{l}{\textit{Pretrained model: $\text{ResNet-50}_{\text{SwAV}}$}} \\
     $k$-NN & n/a & 0.832 & 0.497 & 0.693 & 0.754 & 0.728 &  0.726 & 0.535 & 0.604 \\
      FFT & 736 & \textbf{0.965} & \textbf{0.817} & \textbf{0.736} & \textbf{0.977} & \textbf{0.987} & \textbf{0.892} & \textbf{0.623} & \textbf{0.675} \\
     Mask & 23 & 0.962 & 0.798 & 0.709 & 0.974 & 0.967 & 0.863 & 0.482 & 0.628 \\
   \midrule
\multicolumn{6}{l}{\textit{Pretrained model: $\text{ViT-B/32}_{\text{CLIP}}$}} \\
     $k$-NN & n/a & {0.909}  & {0.694} & {0.666} & {0.858} & {0.818} & {0.768} & \textbf{0.687} & {0.753} \\
     FFT &  2752 & 0.958 & 0.821 & 0.723 & \textbf{0.979} & \textbf{0.974} & 0.885 & 0.640 & 0.809 \\
     Mask & 86 & \textbf{0.971} & \textbf{0.834} & \textbf{0.738} & 0.978 & 0.973 & \textbf{0.891} & 0.668 & \textbf{0.815} \\
    \bottomrule
    \end{tabular}
    \end{table*}
     


\subsection{Datasets and implementation}

We compare the performance of the found subnetworks with standard full-fine-tuning in the supervised and self-supervised adaptation setting.
In all our experiments we use SGD with a momentum of 0.9 and a batch size of 64.
For finding the subnetworks, we set the hyperparameters $\mu=0$, $S_0=1$ and use no weight decay, this configuration is used for all models, in all settings. We run experiments on the default ResNet-18 and ResNet-50 models from the TIMM repository \cite{rw2019timm} (supervised pretraining), SwAV \cite{caronUnsupervisedLearningVisual2021} and DINO \cite{caronEmergingPropertiesSelfSupervised2021} (self-supervised pretraining), as well as Vision Transformers (ViT-B/32) from CLIP \cite{radford2021learning} (cross-modal contrastive pretraining). We use the datasets from \cite{Loedeman2022prompt}.

\paragraph{Supervised ResNets.}
For the supervised adaptation baseline on the ImageNet-pretrained ResNet18 model from the TIMM repository, we use the most common hyperparameters from \cite{salmanAdversariallyRobustImageNet2020} for every dataset on that model architecture, which is a learning rate of 0.001 and a weight decay of 5e-4. However, we use a simple cosine learning rate decay rather than a step decay. For the SwAV-pretrained Resnet-50 model, we use a learning rate of 0.15 and a weight decay of 1e-6. The learning rate was determined by taking the original learning rate the model was trained on and scaling it by the new batch size ($0.15=0.6\cdot64/256$). For finding the subnetworks of these models, we use a learning rate of 50 and a cosine learning rate decay with a linear warmup up to epoch 40. Both for finding the subnetworks and for the baselines of these models, we train with a standard cross-entropy loss and a new random linear head for 150 epochs. We only train or mask the convolutional and downsampling layers.

\paragraph{Supervised Transformers.}
For the CLIP-pretrained vision transformer baseline and linear probe, we use the results from \cite{bahngExploringVisualPrompts2022,Loedeman2022prompt}, which uses the original head from the CLIP model and a cosine similarity metric between the image and text embeddings as the output logits, before using a cross-entropy loss.
For finding the subnetworks of this model, we use the same loss and number of epochs (32) but stick to a simple cosine learning rate decay schedule with a learning rate of 10. We only mask the projections and the Multilayer Perceptrons (MLP) after each attention block.
All supervised experiments use data augmentations from~\cite{Loedeman2022prompt}.

\paragraph{Self-Supervised ResNets.}
The self-supervised experiments are done using the augmentations and other default settings from SwAV \cite{caronEmergingPropertiesSelfSupervised2021}, except with a learning rate of 0.15, batch size of 64, no warm up and 500 prototypes. For finding subnetworks with self-supervision, we keep the prototypes and linear head trainable using the aforementioned hyperparameters (which are thrown away regardless), we only find a mask for the backbone, using the same parameters as for the supervised experiments. We train for 150 epochs and start the queue after 30 epochs.


\paragraph{Model Cascade.} For the dispatcher in the Model Cascade, we train \textsc{cifar100} for 150 epochs and \textsc{inat500} for 54 epochs, which corresponds to approximately the same number of steps. The experts are also trained for approximately the same number of steps ($\approx117188$), based on the formula: $\mathcal{E} = 50000 / \mathcal{D} * 150$, where $\mathcal{D}$ is the dataset size, which varies for each cluster, and $\mathcal{E}$ is the number of epochs, rounded to the nearest integer. We start the queue after the first 1/5 of the epochs. For the router $R$ we choose a 5-way Gaussian Mixture Model applied on dimensionality-reduced feature vectors using the 20 largest components of a Principal Component Analysis.

\begin{table*}
    \caption{\textbf{Comparison of our masking mechanism (SMN) with Ramanujan~\etal~\cite{ramanujanWhatHiddenRandomly2020}'s topk\% method.}
      For completeness, we run the method of Ramanujan~\etal
      with the sparsity-level that was found by our method.
      We report top-1 accuracy on different downstream tasks using different evaluation methods.
      In the last row, we report the sparsity levels found by our method on the different datasets.
      All numbers are reported by masking with a self-supervised loss, starting from a ResNet-50 pretrained with SwAV.
      \\}
    \label{tab:rama}
    
\centering
\small
  \setlength{\tabcolsep}{3pt}
     \begin{tabular}{@{}l c cccccccc}
     \toprule
     Method & Sparsity & \textsc{cifar10} & \textsc{cifar100} & \textsc{dtd} & \textsc{eurosat} & \textsc{flowers} & \textsc{sun397} & \textsc{ucf101} \\
     \midrule
\multicolumn{6}{l}{\textit{$k$-NN evaluation}} \\
        topk\% & 50\% & 0.891	& 0.564 &	0.465 &	0.968 &	0.403	&	0.460	& 0.439 \\ %
        topk\% & found & 0.908	& 0.630 &	0.671 &	\textbf{0.973} &	0.903	 &	\textbf{0.519}	& \textbf{0.572} \\
        SMN & found & \textbf{0.921}	& \textbf{0.656} &	\textbf{0.674} & 0.971 &	\textbf{0.920}	&	0.518	& 0.549 \\
     \midrule
\multicolumn{6}{l}{\textit{Linear probe evaluation}} \\
         topk\% & 50\% & 0.913	& 0.678 & 0.536 & 0.980 & 0.800 & 0.565 & 0.548 \\
         topk\% & found & 0.940	& 0.747 & \textbf{0.733} & \textbf{0.984} & \textbf{0.985} & 0.630 & 0.690 \\
         SMN & found & \textbf{0.950} & \textbf{0.769} & 0.714 & 0.983 & 0.983 & \textbf{0.631} & \textbf{0.697} \\
        \midrule 
        \midrule
                \multicolumn{2}{l}{Sparsity level:} & 91.4\% & 92.5\% & 98.8\% & 96.7\% & 98.6\% & 94.4\% & 95.7\% \\
    \bottomrule
    \end{tabular}
\end{table*} 

\paragraph{Evaluation.}
Linear probes are done with logistic regression on the unaugmented training set embeddings, as done by \cite{caronEmergingPropertiesSelfSupervised2021}. $k$-NN evaluations are done with the default settings from the same paper, with 200 neighbors and a temperature of $0.1$. Source code is available at: \url{https://github.com/alvitawa/UnsupervisedMasking}.




\subsection{Masking is a viable adaptation strategy}
In this section, we validate that masking weights of a pretrained network is a viable solution to adapt it to a given downstream task by comparing our masking adaptation method to more standard transfer techniques such as nearest neighbors classification ($k$-NN) and full fine-tuning of the model parameters (FFT).
For this experiment, we consider masking with a supervised loss to prevent confounding factors compared with full fine-tuning.
For each of these methods, in addition to reporting the accuracy obtained of the downstream tasks, we also report the memory storage required to deploy the transfer technique.
For $k$-NN, we consider that the memory requirement is none.
However, one could argue that it is still required to store the embeddings of the training set of the dataset, which can quickly be at the order of magnitude of the gigabit for large datasets such as iNaturalist and embeddings of high dimension like $\text{ViT-B/32}_{\text{CLIP}}$.

Results comparing $k$-NN, full fine-tuning (FFT) and supervised masking adaptation (Mask) strategies are reported in \cref{tab:masking_sup}.
We observe that masking produces good transfer performance on a large set of downstream tasks, and is even competitive with full fine-tuning.
We see in \cref{tab:masking_sup} that this observation is consistent both across network architectures (convolutional neural network and vision transformers) and pretraining paradigms (supervised, self-supervised and image-text contrastive).
Overall, results in \cref{tab:masking_sup} show that even though not widely adopted by the deep learning community, masking is a viable and storage-efficient adaptation technique.
In addition, compared to full fine-tuning it requires 32 times less memory storage to deploy when looking at the uncompressed number of bits. However, in practice it can require 78.8 times less memory to store the masks due to the fact that the masks are more easily compressible than a full copy of model weights. This is shown experimentally in Appendix \ref{app:cmprs}.


\subsection{Self-supervised adaptation with self-masking}

\begin{figure*}
    \vspace{1cm}
    \begin{subfigure}{.35\linewidth}
        \centering
        \begin{tikzpicture}
        \begin{semilogxaxis}[
            width=5cm,
            xlabel={Label percentage (\%)},
            ylabel={Accuracy (ratio)},
            xtick=data,
            xticklabels={.25, .5, 1, 2, 4, 10},
            x dir=reverse,
            legend pos=south west,
            grid=major,
        ]
        
        \addplot coordinates {
            (.25,0.2107) (.5,0.2794) (1,0.3846) (2,0.4915) (4,0.5714) (10,0.6516)
        };
        
        \addplot coordinates {
            (.25,0.2839) (.5,0.3744) (1,0.4747) (2,0.5666) (4,0.6259) (10,0.6877)
        };
        

        \addplot coordinates {
            (.25,0.1084) (.5,0.1997) (1,0.3226) (2,0.4308) (4,0.5433) (10,0.6417)
        };


        \addplot coordinates {
            (.25,0.015) (.5,0.0389) (1,0.1079) (2,0.2204) (4,0.362) (10,0.5269)
        };
        
        \end{semilogxaxis}
        \end{tikzpicture}

        \caption{\textsc{CIFAR100}}\label{fig:SPARC100line}
    \end{subfigure}
    \hfill
    \begin{subfigure}{.3\linewidth}
        \centering
        
        \begin{tikzpicture}
            \begin{semilogxaxis}[
            width=5cm,
            xlabel={Label percentage (\%)},
            xtick=data,
            xticklabels={.25, .5, 1, 2, 4, 10},
            x dir=reverse,
            legend style={overlay, at={(0.5,1.35)},anchor=north, legend columns=-1 },
            grid=major,
            ]
            
            \addplot coordinates {
                (.25,0.6994) (.5,0.7729) (1,0.8223) (2,0.8494) (4,0.8719) (10,0.8921)
            };
            \addlegendentry{Linear Probe}
            
            \addplot coordinates {
                (.25,0.8466) (.5,0.8857) (1,0.9014) (2,0.9206) (4,0.9312) (10,0.94)
            };
            \addlegendentry{Self-Masking Network+LP}
            

            \addplot coordinates {
                (10, .91) (4, .8687) (2, .843) (1, .7811) (0.5, .7071) (0.25, .5816)
            };
            \addlegendentry{Full Fine-Tuning}

            \addplot coordinates {
                (10, 0.883) (4, .8284) (2, .7746) (1, .6861) (0.5, .5379) (0.25, .3143)
            };
            \addlegendentry{Masking}
            \end{semilogxaxis}
        \end{tikzpicture}

        \caption{\textsc{CIFAR10}}\label{fig:sparC10line}
    \end{subfigure}
    \hfill
    \begin{subfigure}{.3\linewidth}
        \centering
        
        \begin{tikzpicture}
            \begin{axis}[
            width=5cm,
            xlabel={Label percentage (\%)},
            xtick=data,
            xticklabels={10, 9, 8, 7, 6, 5, 4},
            x dir=reverse,
            grid=major,
            ]
            
            \addplot coordinates {
                (10, 0.418) (9, 0.401058) (8, .379295) (7, .362771) (6, .340806) (5, .313552) (4, .270982)
            };
            
            \addplot coordinates {
                (10, 0.4336) (9, 0.420554) (8, .404937) (7, .388917) (6, .371335) (5, .347154) (4, .313401)
            };
            

            \addplot coordinates {
                (10, 0.287) (9, 0.341008) (8, 0.322116) (7, .302065) (6, .278489) (5, .257028) (4, .219345)
            };

            \addplot coordinates {
                (10, 0.0861965) (9, 0.112343) (8, 0.0986902) (7, .0869521) (6, .0731486) (5, .0654912) (4, .049068)
            };
            \end{axis}
        \end{tikzpicture}

        \caption{\textsc{SUN397}}\label{fig:sparC10line}
    \end{subfigure}
        \caption{\textbf{Low-shot adaptation with self-supervised self-masking.} We report top-1 accuracy after transferring to downstream tasks in a low-shot setting (\% of labeled data used).
        We compare different adaptation techniques: linear probing, full fine-tuning, self-masking with a supervised objective and self-masking in a self-supervised manner.
        The pretrained network is $\text{ResNet-50}_{\text{SwAV}}$.
        \\\\}
    \label{fig:lowshot}
\end{figure*}
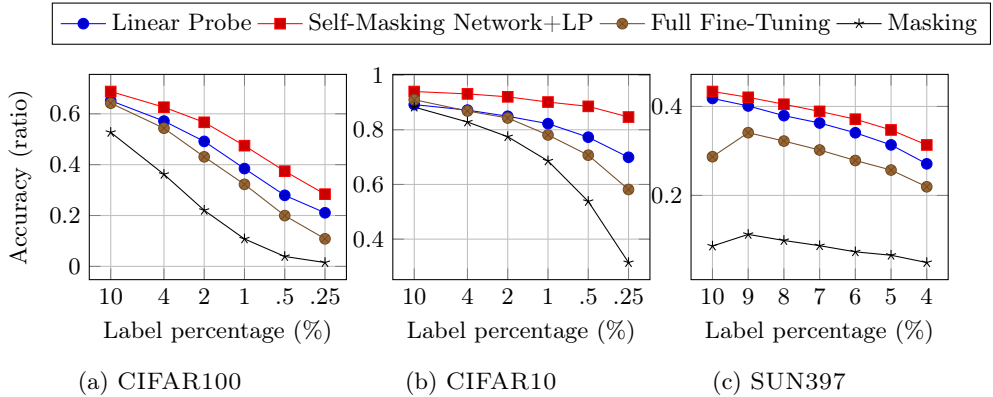
 
Though efficient on the memory storage axis, masking, as well as other typical transfer learning approaches, arguably lack efficiency in terms of \emph{label} utilization.
Indeed, transferring a network to a downstream task is in many cases a process requiring access to many labels in the downstream domain to achieve good performance.
In this section, we explore adaptation to a downstream domain in a label-efficient manner, by learning the transfer mask without using any labels.

\paragraph{Comparison of different masking strategies.}
In \cref{tab:rama}, we compare the performance of our Self-Masking Network (SMN) with our masking method and that of Ramanujan~\etal~\cite{ramanujanWhatHiddenRandomly2020}'s topk\% method, which activates a fixed percentage of the weights with the highest scores rather just weights with scores larger than the threshold. We find that our masking method (threshold masking) outperforms topk\% masking when the sparsity is set to 50\%, which is the sparsity level they use in their paper. We also find that if we use the sparsity which was automatically found by our method, their method can come close to and is sometimes able to outperform our method. However, since Ramanujan~\etal's method is not able to find this sparsity by itself, we argue that our masking method is more practical. 


\paragraph{Low-shot transfer.}

We propose to transfer to different downstream tasks by using only a fraction of the training labels, while keeping access to all training \emph{images}.
Results are presented in \cref{fig:lowshot}.
We observe that self-supervised self-masking outperforms conventional fine-tuning methods in all three datasets tested under low-shot conditions.
We also observe in \cref{fig:lowshot} that the difference in performance is maintained or even exacerbated as fewer labels are used.
This is likely due to the SMN's ability to utilize the unlabeled data for downstream adaptation.
Overall, self-supervised masking is able to overcome the shortcomings of supervised masking in situations where few labeled data is available, providing a solution for both label-efficient and storage-thrifty downstream task transfer.

\subsection{Model Cascade}

\begin{table*}[]
      \caption{\textbf{Self-masking cascade.}
      We compare vanilla self-masking with conditional and unconditional self-masking cascades.
      We also report the (uncompressed) storage requirements relative to the cost of storing the original pretrained model in 32-bit floating point, which we denote by $\beta$.
      We exclude the cost of storing the original model (of size $\beta$) as well as the PCA and GMM parameters, which are negligible. The pretrained backbone used is $\text{ResNet-50}_{\text{SwAV}}$.\\
      }
    \label{tab:supermask}
    \centering
    \begin{tabular}{l c c c}
        \midrule
        Method & Storage & \textsc{CIFAR100} & \textsc{iNat500} \\
        \midrule
        \multicolumn{2}{l}{\textit{$k$-NN evaluation}} \\
        Self-masking & $\beta / 32$ & 0.656 & 0.263 \\
        Self-masking + cascade (conditional) & $6 \beta / 32$ &  0.752 & 0.395 \\
        Self-masking + cascade (unconditional) & $6 \beta / 32$ &  \textbf{0.778} & \textbf{0.424}  \\
        \midrule
        \multicolumn{2}{l}{\textit{Linear probe evaluation}} \\
        Self-masking & $\beta / 32$ & 0.769 & 0.524 \\
        Self-masking + cascade (conditional) & $6 \beta / 32$ &  0.793 & 0.521 \\
        Self-masking + cascade (unconditional) & $6 \beta / 32$ &  \textbf{0.807} & \textbf{0.550}  \\
        \midrule
    \end{tabular}
\end{table*} 
In Table~\ref{tab:supermask}, we evaluate our cascade SMN's.
We observe that the cascade mechanism allows for improving the performance of self-supervised masking by large margins, \ie $+3.8$ points on \textsc{CIFAR-100} and $+2.6$ points on \textsc{iNat500} in linear probing evaluation.
This comes at the cost of an increase in storage requirement ($\times6$).
However, the storage requirement is still significantly ($\times5$) smaller than alternative transfer techniques like full finetuning, and does not use supervision for domain adaptation.
Finally, we also observe that the unconditional variant of the cascade performs better than the conditional model.
An explanation is that the experts learned from neighboring domains still capture useful features that are not completely orthogonal to each other and so benefit overall performances.
Overall, this experiment shows that masking the same model multiple times using self-supervision through our cascade mechanism enables more information to be extracted from the training set, resulting in better downstream accuracy.

    \section{Conclusion}
    




In this work, we started with the problem of adapting a pretrained model to a novel domain without knowledge of the final task or other manual annotations.
To this end, we proposed Self-Masking, a simple approach for learning lightweight binary masks in an unsupervised manner.
We have shown the benefits of Self-Masking are particularly pronounced for the important semi-supervised setting, where both large amounts of unlabeled data and small amounts of labeled data are available.
Finally, we have shown how our approach can be used to produce high-performing model cascades, by selecting domains and training expert models without any supervision. 
We believe this work will gain importance with the increase in parameter counts of vision models.
 

    \setcounter{table}{4}
    \setcounter{figure}{4}   
    \setcounter{equation}{5} 

\newpage
\bibliographystyle{splncs04}
\bibliography{045-main}

\begin{thebibliography}{10}
\providecommand{\url}[1]{\texttt{#1}}
\providecommand{\urlprefix}{URL }
\providecommand{\doi}[1]{https://doi.org/#1}

\bibitem{asano2020self}
Asano, Y.M., Rupprecht, C., Vedaldi, A.: Self-labelling via simultaneous clustering and representation learning. In: International Conference on Learning Representations (ICLR) (2020)

\bibitem{bahngExploringVisualPrompts2022}
Bahng, H., Jahanian, A., Sankaranarayanan, S., Isola, P.: Exploring visual prompts for adapting large-scale models. arXiv preprint arXiv:2203.17274  (2022)

\bibitem{bahng2022visual}
Bahng, H., Jahanian, A., Sankaranarayanan, S., Isola, P.: Visual prompting: Modifying pixel space to adapt pre-trained models. arXiv preprint arXiv:2203.17274  (2022)

\bibitem{ben-zaken-etal-2022-bitfit}
Ben~Zaken, E., Goldberg, Y., Ravfogel, S.: {B}it{F}it: Simple parameter-efficient fine-tuning for transformer-based masked language-models. In: Proceedings of the 60th Annual Meeting of the Association for Computational Linguistics (Volume 2: Short Papers). pp.~1--9. Association for Computational Linguistics, Dublin, Ireland (May 2022). \doi{10.18653/v1/2022.acl-short.1}, \url{https://aclanthology.org/2022.acl-short.1}

\bibitem{bossard2014food}
Bossard, L., Guillaumin, M., Van~Gool, L.: Food-101--mining discriminative components with random forests. In: Computer Vision--ECCV 2014: 13th European Conference, Zurich, Switzerland, September 6-12, 2014, Proceedings, Part VI 13. pp. 446--461. Springer (2014)

\bibitem{cai2020tinytl}
Cai, H., Gan, C., Zhu, L., Han, S.: Tinytl: Reduce memory, not parameters for efficient on-device learning. Advances in Neural Information Processing Systems  \textbf{33},  11285--11297 (2020)

\bibitem{caron2018deep}
Caron, M., Bojanowski, P., Joulin, A., Douze, M.: Deep clustering for unsupervised learning of visual features. In: Proceedings of the European conference on computer vision (ECCV). pp. 132--149 (2018)

\bibitem{caronUnsupervisedLearningVisual2021}
Caron, M., Misra, I., Mairal, J., Goyal, P., Bojanowski, P., Joulin, A.: Unsupervised learning of visual features by contrasting cluster assignments. Advances in Neural Information Processing Systems (NeurIPS)  (2020)

\bibitem{caronEmergingPropertiesSelfSupervised2021}
Caron, M., Touvron, H., Misra, I., J{\'e}gou, H., Mairal, J., Bojanowski, P., Joulin, A.: Emerging properties in self-supervised vision transformers. In: Proceedings of the International Conference on Computer Vision (ICCV) (2021)

\bibitem{chattopadhyay2020learning}
Chattopadhyay, P., Balaji, Y., Hoffman, J.: Learning to balance specificity and invariance for in and out of domain generalization. In: Computer Vision--ECCV 2020: 16th European Conference, Glasgow, UK, August 23--28, 2020, Proceedings, Part IX 16. pp. 301--318. Springer (2020)

\bibitem{cimpoi2014sammy}
Cimpoi, M., Maji, S., Kokkinos, I.: Sammy, mohamed, and andrea vedaldi. Describing textures in the, wild. In, CVPR  \textbf{2} (2014)

\bibitem{cole2022label}
Cole, E., Wilber, K., Van~Horn, G., Yang, X., Fornoni, M., Perona, P., Belongie, S., Howard, A., Aodha, O.M.: On label granularity and object localization. In: Computer Vision--ECCV 2022: 17th European Conference, Tel Aviv, Israel, October 23--27, 2022, Proceedings, Part X. pp. 604--620. Springer (2022)

\bibitem{courbariaux2015binaryconnect}
Courbariaux, M., Bengio, Y., David, J.P.: Binaryconnect: Training deep neural networks with binary weights during propagations. Advances in neural information processing systems  \textbf{28} (2015)

\bibitem{courbariaux2016binarized}
Courbariaux, M., Hubara, I., Soudry, D., El-Yaniv, R., Bengio, Y.: Binarized neural networks: Training deep neural networks with weights and activations constrained to+ 1 or-1. arXiv preprint arXiv:1602.02830  (2016)

\bibitem{cuturi2013sinkhorn}
Cuturi, M.: Sinkhorn distances: Lightspeed computation of optimal transport. Advances in neural information processing systems  \textbf{26} (2013)

\bibitem{dehghani2023scaling}
Dehghani, M., Djolonga, J., Mustafa, B., Padlewski, P., Heek, J., Gilmer, J., Steiner, A., Caron, M., Geirhos, R., Alabdulmohsin, I., et~al.: Scaling vision transformers to 22 billion parameters. arXiv preprint arXiv:2302.05442  (2023)

\bibitem{elsayed2018adversarial}
Elsayed, G.F., Goodfellow, I., Sohl-Dickstein, J.: Adversarial reprogramming of neural networks. arXiv preprint arXiv:1806.11146  (2018)

\bibitem{ericsson2021well}
Ericsson, L., Gouk, H., Hospedales, T.M.: How well do self-supervised models transfer? In: Proceedings of the IEEE/CVF Conference on Computer Vision and Pattern Recognition. pp. 5414--5423 (2021)

\bibitem{evci2022head2toe}
Evci, U., Dumoulin, V., Larochelle, H., Mozer, M.C.: Head2toe: Utilizing intermediate representations for better transfer learning. In: International Conference on Machine Learning. pp. 6009--6033. PMLR (2022)

\bibitem{fangSEEDSelfsupervisedDistillation2021}
Fang, Z., Wang, J., Wang, L., Zhang, L., Yang, Y., Liu, Z.: Seed: Self-supervised distillation for visual representation. International Conference on Learning Representations (ICLR)  (2021)

\bibitem{frankleLotteryTicketHypothesis2019}
Frankle, J., Carbin, M.: The lottery ticket hypothesis: Finding sparse, trainable neural networks. International Conference on Learning Representations (ICLR)  (2019)

\bibitem{gao2021clip}
Gao, P., Geng, S., Zhang, R., Ma, T., Fang, R., Zhang, Y., Li, H., Qiao, Y.: Clip-adapter: Better vision-language models with feature adapters. arXiv preprint arXiv:2110.04544  (2021)

\bibitem{gururanganDonStopPretraining2020}
Gururangan, S., Marasović, A., Swayamdipta, S., Lo, K., Beltagy, I., Downey, D., Smith, N.A.: Don't stop pretraining: Adapt language models to domains and tasks. Proceedings of ACL  (2020)

\bibitem{he2022masked}
He, K., Chen, X., Xie, S., Li, Y., Doll{\'a}r, P., Girshick, R.: Masked autoencoders are scalable vision learners. In: Proceedings of the IEEE/CVF Conference on Computer Vision and Pattern Recognition. pp. 16000--16009 (2022)

\bibitem{he2020momentum}
He, K., Fan, H., Wu, Y., Xie, S., Girshick, R.: Momentum contrast for unsupervised visual representation learning. In: Proceedings of the IEEE/CVF conference on computer vision and pattern recognition. pp. 9729--9738 (2020)

\bibitem{helber2019eurosat}
Helber, P., Bischke, B., Dengel, A., Borth, D.: Eurosat: A novel dataset and deep learning benchmark for land use and land cover classification. IEEE Journal of Selected Topics in Applied Earth Observations and Remote Sensing  \textbf{12}(7),  2217--2226 (2019)

\bibitem{hoefler2021sparsity}
Hoefler, T., Alistarh, D., Ben-Nun, T., Dryden, N., Peste, A.: Sparsity in deep learning: Pruning and growth for efficient inference and training in neural networks. The Journal of Machine Learning Research  \textbf{22}(1),  10882--11005 (2021)

\bibitem{howard2018universal}
Howard, J., Ruder, S.: Universal language model fine-tuning for text classification. arXiv preprint arXiv:1801.06146  (2018)

\bibitem{hulora}
Hu, E.J., Shen, Y., Wallis, P., Allen-Zhu, Z., Li, Y., Wang, S., Wang, L., Chen, W.: Lora: Low-rank adaptation of large language models. International Conference on Learning Representations (ICLR)  (2021)

\bibitem{jia2022visual}
Jia, M., Tang, L., Chen, B.C., Cardie, C., Belongie, S., Hariharan, B., Lim, S.N.: Visual prompt tuning. In: Computer Vision--ECCV 2022: 17th European Conference, Tel Aviv, Israel, October 23--27, 2022, Proceedings, Part XXXIII. pp. 709--727. Springer (2022)

\bibitem{kloberdanz2021improved}
Kloberdanz, E., Tian, J., Le, W.: An improved (adversarial) reprogramming technique for neural networks. In: Artificial Neural Networks and Machine Learning--ICANN 2021: 30th International Conference on Artificial Neural Networks, Bratislava, Slovakia, September 14--17, 2021, Proceedings, Part I 30. pp. 3--15. Springer (2021)

\bibitem{krizhevsky2009learning}
Krizhevsky, A., Hinton, G., et~al.: Learning multiple layers of features from tiny images  (2009)

\bibitem{lagunasBlockPruningFaster2021}
Lagunas, F., Charlaix, E., Sanh, V., Rush, A.M.: Block pruning for faster transformers. Empirical Methods in Natural Language Processing  (2021)

\bibitem{liTernaryWeightNetworks2022}
Liu, B., Li, F., Wang, X., Zhang, B., Yan, J.: Ternary weight networks. In: ICASSP 2023-2023 IEEE International Conference on Acoustics, Speech and Signal Processing (ICASSP). pp.~1--5. IEEE (2023)

\bibitem{Loedeman2022prompt}
Loedeman, J., Stol, M., Han, T., Asano, Y.M.: Prompt generation networks for efficient adaptation of frozen vision transformers. arxiv preprint arxiv:2210.06466  (2022)

\bibitem{malachProvingLotteryTicket}
Malach, E., Yehudai, G., Shalev-Schwartz, S., Shamir, O.: Proving the lottery ticket hypothesis: Pruning is all you need. In: International Conference on Machine Learning. pp. 6682--6691. PMLR (2020)

\bibitem{mallya2018piggyback}
Mallya, A., Davis, D., Lazebnik, S.: Piggyback: Adapting a single network to multiple tasks by learning to mask weights. In: Proceedings of the European Conference on Computer Vision (ECCV). pp. 67--82 (2018)

\bibitem{mallyaPackNetAddingMultiple2018}
Mallya, A., Lazebnik, S.: Packnet: Adding multiple tasks to a single network by iterative pruning. In: Proceedings of the IEEE conference on Computer Vision and Pattern Recognition. pp. 7765--7773 (2018)

\bibitem{nilsback2008automated}
Nilsback, M.E., Zisserman, A.: Automated flower classification over a large number of classes. In: 2008 Sixth Indian Conference on Computer Vision, Graphics \& Image Processing. pp. 722--729. IEEE (2008)

\bibitem{parkhi2012cats}
Parkhi, O.M., Vedaldi, A., Zisserman, A., Jawahar, C.: Cats and dogs. In: 2012 IEEE conference on computer vision and pattern recognition. pp. 3498--3505. IEEE (2012)

\bibitem{radford2021learning}
Radford, A., Kim, J.W., Hallacy, C., Ramesh, A., Goh, G., Agarwal, S., Sastry, G., Askell, A., Mishkin, P., Clark, J., et~al.: Learning transferable visual models from natural language supervision. In: International conference on machine learning. pp. 8748--8763. PMLR (2021)

\bibitem{ramanujanWhatHiddenRandomly2020}
Ramanujan, V., Wortsman, M., Kembhavi, A., Farhadi, A., Rastegari, M.: What's hidden in a randomly weighted neural network? In: Proceedings of the Conference on Computer Vision and Pattern Recognition (CVPR) (2020)

\bibitem{rebuffi2017learning}
Rebuffi, S.A., Bilen, H., Vedaldi, A.: Learning multiple visual domains with residual adapters. Advances in neural information processing systems  \textbf{30} (2017)

\bibitem{reedSelfSupervisedPretrainingImproves2022}
Reed, C.J., Yue, X., Nrusimha, A., Ebrahimi, S., Vijaykumar, V., Mao, R., Li, B., Zhang, S., Guillory, D., Metzger, S., et~al.: Self-supervised pretraining improves self-supervised pretraining. In: Proceedings of the IEEE/CVF Winter Conference on Applications of Computer Vision. pp. 2584--2594 (2022)

\bibitem{salmanAdversariallyRobustImageNet2020}
Salman, H., Ilyas, A., Engstrom, L., Kapoor, A., Madry, A.: Do adversarially robust imagenet models transfer better? Advances in Neural Information Processing Systems (NeurIPS)  (2020)

\bibitem{sanhMovementPruningAdaptive}
Sanh, V., Wolf, T., Rush, A.: Movement pruning: Adaptive sparsity by fine-tuning. Advances in Neural Information Processing Systems (NeurIPS)  (2020)

\bibitem{soomro2012ucf101}
Soomro, K., Zamir, A.R., Shah, M.: Ucf101: A dataset of 101 human actions classes from videos in the wild. arXiv preprint arXiv:1212.0402  (2012)

\bibitem{tian2021divide}
Tian, Y., Henaff, O.J., van~den Oord, A.: Divide and contrast: Self-supervised learning from uncurated data. In: Proceedings of the IEEE/CVF International Conference on Computer Vision. pp. 10063--10074 (2021)

\bibitem{van2018inaturalist}
Van~Horn, G., Mac~Aodha, O., Song, Y., Cui, Y., Sun, C., Shepard, A., Adam, H., Perona, P., Belongie, S.: The inaturalist species classification and detection dataset. In: Proceedings of the IEEE conference on computer vision and pattern recognition. pp. 8769--8778 (2018)

\bibitem{rw2019timm}
Wightman, R.: Pytorch image models. \url{https://github.com/rwightman/pytorch-image-models} (2019). \doi{10.5281/zenodo.4414861}

\bibitem{wortsmanDiscoveringNeuralWirings2019}
Wortsman, M., Farhadi, A., Rastegari, M.: Discovering neural wirings. Advances in Neural Information Processing Systems (NeurIPS)  (2019)

\bibitem{wortsman2022model}
Wortsman, M., Ilharco, G., Gadre, S.Y., Roelofs, R., Gontijo-Lopes, R., Morcos, A.S., Namkoong, H., Farhadi, A., Carmon, Y., Kornblith, S., et~al.: Model soups: averaging weights of multiple fine-tuned models improves accuracy without increasing inference time. In: International Conference on Machine Learning. pp. 23965--23998. PMLR (2022)

\bibitem{wortsmanSupermasksSuperposition}
Wortsman, M., Ramanujan, V., Liu, R., Kembhavi, A., Rastegari, M., Yosinski, J., Farhadi, A.: Supermasks in superposition. Advances in Neural Information Processing Systems (NeurIPS)  (2020)

\bibitem{xiao2010sun}
Xiao, J., Hays, J., Ehinger, K.A., Oliva, A., Torralba, A.: Sun database: Large-scale scene recognition from abbey to zoo. In: 2010 IEEE computer society conference on computer vision and pattern recognition. pp. 3485--3492. IEEE (2010)

\bibitem{zhang2020side}
Zhang, J.O., Sax, A., Zamir, A., Guibas, L., Malik, J.: Side-tuning: a baseline for network adaptation via additive side networks. In: Computer Vision--ECCV 2020: 16th European Conference, Glasgow, UK, August 23--28, 2020, Proceedings, Part III 16. pp. 698--714. Springer (2020)

\bibitem{zhang2021tip}
Zhang, R., Fang, R., Zhang, W., Gao, P., Li, K., Dai, J., Qiao, Y., Li, H.: Tip-adapter: Training-free clip-adapter for better vision-language modeling. arXiv preprint arXiv:2111.03930  (2021)

\bibitem{zhu2016trained}
Zhu, C., Han, S., Mao, H., Dally, W.J.: Trained ternary quantization. International Conference on Learning Representations (ICLR)  (2017)

\end{thebibliography}

\newpage

    \appendix


\clearpage


\section{Verification of Hyperparameter-free masking}
\label{apx:expi}

In practice, it is not feasible to control the conditions perfectly, and randomness from CUDA operations or mini-batch sampling, as well as differences in behavior of floating point arithmetic under different scales will result in slightly different results. We show in this section that performance of the models remain generally unaffected, indicating that these invariances also hold in practice.

Figure \ref{fig:equi} shows an experimental verification of the proofs in Section \ref{sec:hyfr}. It can be seen that both experiments have a very similar loss curve, which is an indication that the models are behaving the same. Note that we controlled the most important random components such as the data sampling, however some randomness from CUDA operations could still be present, in addition to differences in behavior of floating point arithmetic under the different scales. For this reason we cannot expect in practice that both hyperparameter settings produce exactly the same masks.

Notice that it can be seen in Figure \ref{fig:o1o} that only doubling the learning rate gives the same loss curve as only increasing the threshold to $0.5$. This is in line with the theorems posited in the main paper, as increasing the threshold by $0.5$ is equivalent to reducing the score initialization by $0.5$, effectively halving it, which is in turn equivalent to doubling the learning rate.

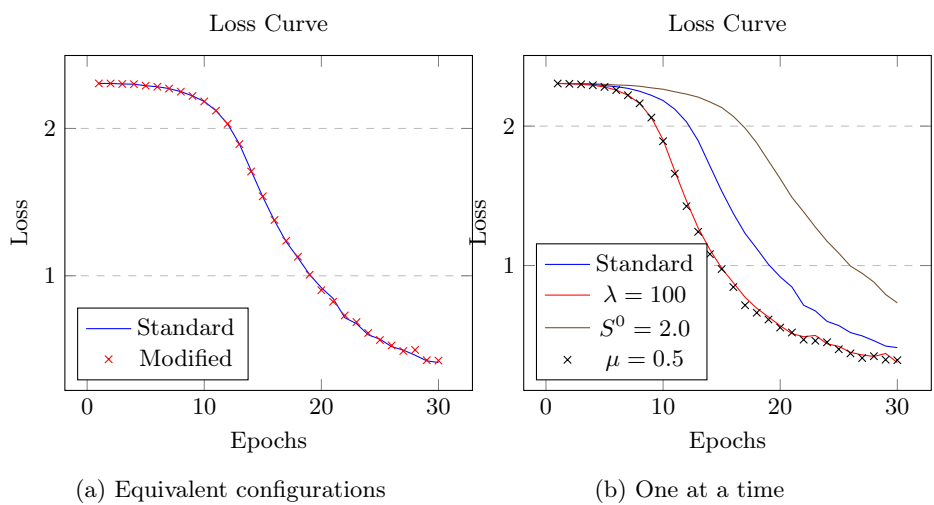
\begin{figure*}
    \centering
    \vspace{1cm}
    \begin{subfigure}{0.5\textwidth}
            \begin{tikzpicture}
            \begin{axis}[
                title={Loss Curve},
                width=7cm,
                xlabel={Epochs},
                ylabel={Loss},
                legend pos=south west,
                ymajorgrids=true,
                grid style=dashed,
            ]

            \addplot[
                color=blue,
                mark=none,
            ]
            coordinates {
            (1,2.305828332901)(2,2.30540108680725)(3,2.30212211608887)(4,2.30046963691711)(5,2.29013347625732)(6,2.28311514854431)(7,2.27093982696533)(8,2.2499372959137)(9,2.22092795372009)(10,2.18349003791809)(11,2.12098121643066)(12,2.030428647995)(13,1.89504551887512)(14,1.7138409614563)(15,1.53481411933899)(16,1.3733549118042)(17,1.23047280311584)(18,1.12144184112549)(19,1.00792455673218)(20,0.918778479099274)(21,0.846426486968994)(22,0.715238213539124)(23,0.674653172492981)(24,0.599408805370331)(25,0.569653153419495)(26,0.51900851726532)(27,0.495268195867538)(28,0.460688203573227)(29,0.422353446483612)(30,0.411089450120926)
            };

            \addplot[
                only marks,
                color=red,
                mark=x
            ]
            coordinates {
            (1,2.305828332901)
            (2,2.30540108680725)
            (3,2.30212211608887)
            (4,2.30046963691711)
            (5,2.29013347625732)
            (6,2.28311514854431)
            (7,2.27093982696533)
            (8,2.2499372959137)
            (9,2.22092795372009)
            (10,2.18349003791809)
            (11,2.12098121643066)
            (12,2.03144717216492)
            (13,1.89374458789825)
            (14,1.7078857421875)
            (15,1.53935134410858)
            (16,1.37850046157837)
            (17,1.23757600784302)
            (18,1.12939822673798)
            (19,1.00702214241028)
            (20,0.901998817920685)
            (21,0.823670327663422)
            (22,0.730610072612763)
            (23,0.685326337814331)
            (24,0.611143171787262)
            (25,0.563665211200714)
            (26,0.52739679813385)
            (27,0.489237159490585)
            (28,0.496811896562576)
            (29,0.426518708467484)
            (30,0.424037992954254)
            };

            \legend{Standard, Modified}
            
            \end{axis}
            \end{tikzpicture}
        \caption{Equivalent configurations}
        \label{fig:equi}
    \end{subfigure}%
    \begin{subfigure}{0.5\textwidth}
            \begin{tikzpicture}
            \begin{axis}[
                title={Loss Curve},
                width=7cm,
                xlabel={Epochs},
                ylabel={Loss},
                legend pos=south west,
                ymajorgrids=true,
                grid style=dashed,
            ]
            
            \addplot[
                color=blue,
                mark=none,
            ]
            coordinates {
            (1,2.305828332901)(2,2.30540108680725)(3,2.30212211608887)(4,2.30046963691711)(5,2.29013347625732)(6,2.28311514854431)(7,2.27093982696533)(8,2.2499372959137)(9,2.22092795372009)(10,2.18349003791809)(11,2.12098121643066)(12,2.030428647995)(13,1.89504551887512)(14,1.7138409614563)(15,1.53481411933899)(16,1.3733549118042)(17,1.23047280311584)(18,1.12144184112549)(19,1.00792455673218)(20,0.918778479099274)(21,0.846426486968994)(22,0.715238213539124)(23,0.674653172492981)(24,0.599408805370331)(25,0.569653153419495)(26,0.51900851726532)(27,0.495268195867538)(28,0.460688203573227)(29,0.422353446483612)(30,0.411089450120926)
            };

            \addplot+[
                mark=none,
            ]
            coordinates {
            (1,2.305828332901)
            (2,2.30342936515808)
            (3,2.299729347229)
            (4,2.29300165176392)
            (5,2.28116321563721)
            (6,2.25851464271545)
            (7,2.22152400016785)
            (8,2.1633780002594)
            (9,2.06109619140625)
            (10,1.9012154340744)
            (11,1.6775951385498)
            (12,1.46240150928497)
            (13,1.26910698413849)
            (14,1.11376762390137)
            (15,0.981271982192993)
            (16,0.876228630542755)
            (17,0.772507786750794)
            (18,0.69322794675827)
            (19,0.631247401237488)
            (20,0.567964315414429)
            (21,0.511747419834137)
            (22,0.489105671644211)
            (23,0.498668044805527)
            (24,0.438872426748276)
            (25,0.419795334339142)
            (26,0.377143710851669)
            (27,0.358603149652481)
            (28,0.347572207450867)
            (29,0.369342625141144)
            (30,0.304389029741287)
            };

            \addplot+[mark=none,]
            coordinates {
            (1,2.305828332901)(2,2.30639433860779)(3,2.30430889129639)(4,2.30375003814697)(5,2.29970741271973)(6,2.29512977600098)(7,2.29252099990845)(8,2.28609967231751)(9,2.27443552017212)(10,2.26463413238525)(11,2.24633288383484)(12,2.22957801818848)(13,2.20796871185303)(14,2.17338395118713)(15,2.1334764957428)(16,2.06994366645813)(17,1.98549067974091)(18,1.87890315055847)(19,1.75041031837463)(20,1.62321305274963)(21,1.48915195465088)(22,1.38504493236542)(23,1.27565467357636)(24,1.17446327209473)(25,1.08860266208649)(26,0.996037483215332)(27,0.943734884262085)(28,0.881288945674896)(29,0.792271018028259)(30,0.732134759426117)
            };


            \addplot+[
                only marks,
                mark=x
            ]
            coordinates {
                (1,2.305828332901)
                (2,2.30342936515808)
                (3,2.299729347229)
                (4,2.29300165176392)
                (5,2.28116321563721)
                (6,2.25851464271545)
                (7,2.22152400016785)
                (8,2.1633780002594)
                (9,2.06136560440063)
                (10,1.89182603359222)
                (11,1.65883266925812)
                (12,1.42556405067444)
                (13,1.24208128452301)
                (14,1.08170640468597)
                (15,0.974703371524811)
                (16,0.845559358596802)
                (17,0.714986622333527)
                (18,0.661494672298431)
                (19,0.612804174423218)
                (20,0.555099904537201)
                (21,0.519690811634064)
                (22,0.468415200710297)
                (23,0.460538595914841)
                (24,0.449539005756378)
                (25,0.399415194988251)
                (26,0.369521588087082)
                (27,0.336500436067581)
                (28,0.349646478891373)
                (29,0.324695020914078)
                (30,0.320863753557205)
            };

            \legend{Standard, $\lambda=100$, $S^0=2.0$, $\mu=0.5$}
            \end{axis}
            \end{tikzpicture}
        \caption{One at a time}
        \label{fig:o1o}
    \end{subfigure}%
    \caption{Left: Comparison of the loss for standard training ($\lambda=50$, $S^0=1.0$, $\mu=0.0$) with equivalent but distinct hyperparameters ($\lambda=100$, $S^0=2.5$, $\mu=0.5$). Shown is the progression of the loss during training. Right: How the curves would differ when applying standard training, except changing only one of the hyperparameters at a time (doubling the learning rate, doubling the score initialization or shifting the threshold with 0.5). These experiments were run on CIFAR-10, with the supervised masking algorithm.}
    \label{fig:experi}
\end{figure*}

\section{Ablations}


\begin{table*}
     \caption{\textbf{Ablations.} We ablate the key components of our Self-Masking Network: the number of prototypes, network initialization, and the layers that are masked. We evaluate via $k$-NN evaluation. The row marked with an asterisk (*) indicates the configuration used in the rest of the paper.\\}
    \label{tab:abla}
    \setlength{\tabcolsep}{7pt}
    	\begin{subtable}[h]{0.5\textwidth}
                \caption{
  Varying prototypes
  }
  \label{tab:abl:protos}
  \centering
    \begin{tabular}{lccc}
    \toprule
    & \textsc{cifar10}  &  \textsc{dtd} &  \textsc{sun397} \\
        \midrule
        50 & 0.510 & 0.644 & 0.503\\
        500* & \textbf{0.921} & \textbf{0.674} & \textbf{0.518} \\
        5000 & 0.920 & 0.640 & 0.496\\
    \bottomrule
  \end{tabular}     	\end{subtable}
    	\begin{subtable}[h]{0.5\textwidth}
                \caption{
  Keeping layers frozen.
  }
  \label{tab:abl:frozen}
  \centering
    \begin{tabular}{lccc}
    \toprule
    &  \textsc{cifar10}  &  \textsc{dtd} &  \textsc{sun397} \\
        \midrule
        none & 0.560 & 0.434 & 0.177 \\
        BNs*  & \textbf{0.921} & \textbf{0.674} & \textbf{0.518} \\
        biases & 0.681 & 0.505 & 0.266 \\
    \bottomrule
  \end{tabular}     	\end{subtable}
     \\
    \hfill
     	\begin{subtable}[h]{1\textwidth}
                \caption{
  Varying the initialiation
}
  \label{tab:abl:init}
  \centering
    \begin{tabular}{lccc}
    \toprule
    &  \textsc{cifar10}  &  \textsc{dtd} &  \textsc{sun397} \\
        \midrule
        DINO & 0.915 & \textbf{0.695} & \textbf{0.529} \\
        SwAV* & \textbf{0.921} & 0.674 & 0.518 \\
        TIMM & 0.900 & 0.623 & 0.505 \\
    \bottomrule
  \end{tabular}     	\end{subtable} 
    \hfill
\end{table*}

In this section, we ablate several components of our model. 

\textbf{Number of prototypes. } We find that as long as our model is provided with enough capacity of 500 or more prototypes, the exact number does not matter and we achieve good performances, echoing previous self-supervised clustering findings~\cite{asano2020self,caron2018deep}

\textbf{Excluding weights from masking.} Finally, in \cref{tab:abl:frozen} we evaluate what happens when the masking strategy is applied to all weights (`none' are frozen), all excluding the batch-norms (our setup) or all except the bias terms. We find that, as in the original masking formulations~\cite{ramanujanWhatHiddenRandomly2020,mallya2018piggyback}, it is essential to not apply masking to the batch-norm statistics, this could be due to the drastic effect zeroing out the batch norm weight can have, which disables a node completely.

\textbf{Pretrained weights.} In \cref{tab:abl:init}, we evaluate whether our self-masking can be run on differently pretrained backbones. 
Despite our loss utilising a self-supervised clustering formulation, we find it performs just as well on a teacher-distillation pretrained method, DINO~\cite{caronEmergingPropertiesSelfSupervised2021} outperforming the supervisedly pretrained TIMM weights.

\section{Mask compression}

It turns out that simply compressing the masks with off-the-shelf approaches can significantly reduce the storage costs. In particular, masks are more easily compressible than neural network weights, which further reduces the storage costs of masks when compared to a full set of fine-tuned weights. We found that, after compressing both, the masks take up 1.305\% of the storage costs when compared to the full-fine-tuned network weights using the models trained on cifar100 (SMN vs self-sup fine-tuning) and 1.2695\% using the models trained on SUN397. If neither are compressed at all, the masks take up 3.125\% the storage cost of the full fine-tuned weights (using f32). See the full experiment results below in Tables 
\ref{tab:cmp1}
 and \ref{tab:cmp2}. Overall, compressing the masks reduces their storage cost by up to 83\%. 

\label{app:cmprs}
\begin{table*}[h]
\caption{Compressing learned masks (using the Self-Masking method) with different off-the-shelf compression methods vs compressing the weights after Full Fine-Tuning (cifar100 dataset).}
\label{tab:cmp1}
\begin{subtable}{.5\linewidth}
\centering
\caption{Masked}
\begin{tabular}{llc}
\toprule
Method & Masks & {Reduction (\%)} \\ 
\midrule
gzip   & 23462592 & 78.32 \\
bz2    & 23462592 & 79.24 \\
lzma   & 23462592 & 80.73 \\
lz4    & 23462592 & 59.06 \\
snappy & 23462592 & 62.57 \\
\bottomrule
\end{tabular}
\end{subtable}%
\begin{subtable}{.5\linewidth}
\centering
\caption{Trained}
\begin{tabular}{llc}
\toprule
Method & f32's & {Reduction (\%)} \\ 
\midrule
gzip   & 23462592 & 6.99 \\
bz2    & 23462592 & 4.69 \\
lzma   & 23462592 & 7.70 \\
lz4    & 23462592 & -0.39 \\
snappy & 23462592 & -0.0046 \\
\bottomrule
\end{tabular}
\end{subtable}
\end{table*}

\begin{table*}[h]
\caption{Same as \cref{tab:cmp1}, but with SUN397 dataset.}
\label{tab:cmp2}
\begin{subtable}{.5\linewidth}
\centering
\caption{Masked}
\begin{tabular}{llc}
\toprule
Method & Masks & {Reduction (\%)} \\ 
\midrule
gzip   & 23462592 & 81.25 \\
bz2    & 23462592 & 82.11 \\
lzma   & 23462592 & 83.32 \\
lz4    & 23462592 & 62.52 \\
snappy & 23462592 & 66.54 \\
\bottomrule
\end{tabular}
\end{subtable}%
\begin{subtable}{.5\linewidth}
\centering
\caption{Trained}
\begin{tabular}{llc}
\toprule
Method & f32's & {Reduction (\%)} \\ 
\midrule
gzip   & 23462592 & 6.99 \\
bz2    & 23462592 & 4.69 \\
lzma   & 23462592 & 7.69 \\
lz4    & 23462592 & -0.39 \\
snappy & 23462592 & -0.0046 \\
\bottomrule
\end{tabular}
\end{subtable}
\end{table*}

\section{Found Sparsities}

In this section we document the sparsities found by our method in detail.

\subsection{Supervised sparsities}

Table \ref{tab:sprs} shows the sparsities found by the supervised masking experiments shown in Table 1
. Notably, the CLIP vision transformer deactivates very few weights, especially for some datasets. An explanation could be the fact that this model was trained with a pre-aligned classification head, thanks to CLIP's textual embeddings, thus requiring less invasive alterations. However, further analysis is necessary to determine whether another factor is responsible, such as the different architecture.

The vision transformer is further analyzed in Figure \ref{fig:vit}. It would appear that the weights in the middle layers generally get masked the most. The biases however do not show a similar pattern. Interestingly though, the biases of the key projections (\textbf{key\_proj}) are left completely unchanged by the algorithm. A similar analysis for the ResNet-50 model can be seen in Figure \ref{fig:daaa}.

\newcommand{\mytablereadperiodic}[1]{%
    \pgfplotstabletypeset[
        col sep = tab,
        string type,
        string replace*={_}{\textsubscript},
        every head row/.style={before row=\toprule,after row=\midrule},
        every last row/.style={after row=\bottomrule},
        every row no 2/.style={after row=\midrule},
        every row no 5/.style={after row=\midrule},
        every row no 9/.style={after row=\midrule},
        skip rows between index={6}{1000}
    ]
    {#1}
}

\begin{table*}
   \setlength{\tabcolsep}{0.3em}
   \captionsetup{aboveskip=5pt,belowskip=5pt}
    \caption{Percentage of weights that are active after masking different architectues in the supervised setting. }
    \label{tab:sprs}
    \centering
    \mytablereadperiodic{./plots/part2_accuracy/full/part2_sparsity.csv}
\end{table*}

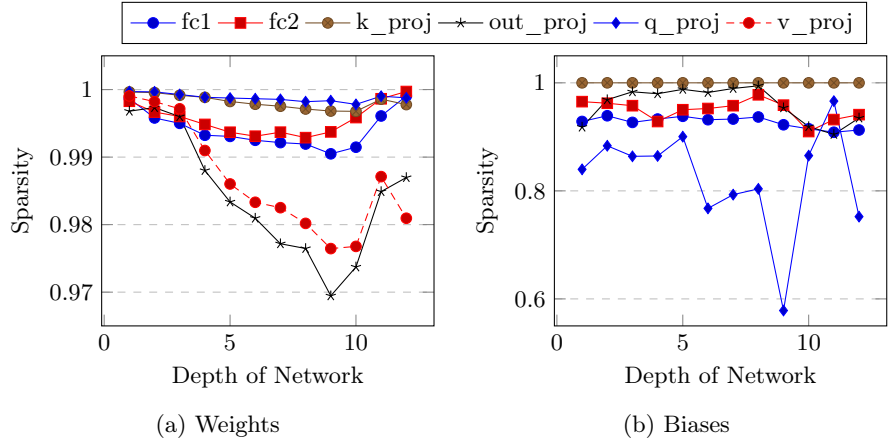
\begin{figure*}
    \centering
    \vspace{1cm}
    \begin{subfigure}{0.5\textwidth}
        \centering
        \begin{tikzpicture}
\begin{axis}[
    width=6cm,
    xlabel={Depth of Network},
    ylabel={Sparsity},
    ymin=0.965, ymax=1.005,
    legend style={overlay, at={(0.06,1.20)},anchor=north west, legend columns=-1 },
    ymajorgrids=true,
    grid style=dashed,
]

\addplot+[
]
coordinates {
    (1, 0.9986860603094101) 
    (2, 0.9957973808050156) 
    (3, 0.9950021058320999) 
    (4, 0.9932392984628677) 
    (5, 0.993068277835846) 
    (6, 0.99249267578125) 
    (7, 0.9921491593122482) 
    (8, 0.9919358491897583) 
    (9, 0.9904863685369492) 
    (10, 0.9914673715829849) 
    (11, 0.9960905611515045) 
    (12, 0.9990468472242355) 
};

\addlegendentry{fc1}

\addplot+[
]
coordinates {
    (1, 0.9982739686965942) 
    (2, 0.9966697841882706) 
    (3, 0.996020644903183) 
    (4, 0.9948420077562332) 
    (5, 0.9936822205781937) 
    (6, 0.9930944442749023) 
    (7, 0.9936909228563309) 
    (8, 0.992881253361702) 
    (9, 0.993736058473587) 
    (10, 0.9958572387695312) 
    (11, 0.9986405968666077) 
    (12, 0.9997123181819916) 
};

\addlegendentry{fc2}

\addplot+[
]
coordinates {
    (1, 0.9996465146541595) 
    (2, 0.9995494484901428) 
    (3, 0.9991556853055954) 
    (4, 0.9988509267568588) 
    (5, 0.9982388764619827) 
    (6, 0.9978315681219101) 
    (7, 0.9975357204675674) 
    (8, 0.9971220344305038) 
    (9, 0.9968210905790329) 
    (10, 0.9967782944440842) 
    (11, 0.9985826313495636) 
    (12, 0.997779443860054) 
};

\addlegendentry{k\_proj}

\addplot+[
]
coordinates {
    (1, 0.9967982172966003) 
    (2, 0.9973377734422684) 
    (3, 0.9959950000047684) 
    (4, 0.9880201667547226) 
    (5, 0.9833624213933945) 
    (6, 0.9809523820877075) 
    (7, 0.9771686345338821) 
    (8, 0.9764637649059296) 
    (9, 0.9694540351629257) 
    (10, 0.9737290441989899) 
    (11, 0.9848912507295609) 
    (12, 0.9869753569364548) 
};

\addlegendentry{out\_proj}

\addplot+[
]
coordinates {
    (1, 0.9997041523456573) 
    (2, 0.9996609091758728) 
    (3, 0.9992714077234268) 
    (4, 0.9988589733839035) 
    (5, 0.9987483620643616) 
    (6, 0.9986500442028046) 
    (7, 0.9985482841730118) 
    (8, 0.9982083737850189) 
    (9, 0.9983571469783783) 
    (10, 0.9978090971708298) 
    (11, 0.9990221709012985) 
    (12, 0.9987775981426239) 
};

\addlegendentry{q\_proj}

\addplot+[
]
coordinates {
    (1, 0.9990590512752533) 
    (2, 0.998213455080986) 
    (3, 0.9971495866775513) 
    (4, 0.9909761399030685) 
    (5, 0.9860216975212097) 
    (6, 0.983311116695404) 
    (7, 0.9825096130371094) 
    (8, 0.9801843464374542) 
    (9, 0.9764493405818939) 
    (10, 0.9767765551805496) 
    (11, 0.9871092885732651) 
    (12, 0.9809468686580658) 
};

\addlegendentry{v\_proj}

    
    
    \end{axis}
    \end{tikzpicture}
         \caption{Weights}
    \end{subfigure}%
    \begin{subfigure}{0.5\textwidth}
        \centering
        
\begin{tikzpicture}
\begin{axis}[
    width=6cm,
    xlabel={Depth of Network},
    ylabel={Sparsity},
    ymin=0.55, ymax=1.05,
    ymajorgrids=true,
    grid style=dashed,
]

\addplot+[
]
coordinates {
(1, 0.9284668117761612) 
(2, 0.9390462338924408) 
(3, 0.9267578274011612) 
(4, 0.932942733168602) 
(5, 0.938069686293602) 
(6, 0.931722030043602) 
(7, 0.932942733168602) 
(8, 0.9367675930261612) 
(9, 0.9223632961511612) 
(10, 0.9149577021598816) 
(11, 0.908528670668602) 
(12, 0.912679061293602) 
};


\addplot+[
]
coordinates {
(1, 0.965169295668602) 
(2, 0.9622395932674408) 
(3, 0.9576823115348816) 
(4, 0.9283854365348816) 
(5, 0.9501953274011612) 
(6, 0.952473983168602) 
(7, 0.9576823115348816) 
(8, 0.9778645932674408) 
(9, 0.9586588740348816) 
(10, 0.9098307490348816) 
(11, 0.9319661557674408) 
(12, 0.9410807490348816) 
};


\addplot+[
]
coordinates {
(1, 1.0) 
(2, 1.0) 
(3, 1.0) 
(4, 1.0) 
(5, 1.0) 
(6, 1.0) 
(7, 1.0) 
(8, 1.0) 
(9, 1.0) 
(10, 1.0) 
(11, 1.0) 
(12, 1.0) 
};


\addplot+[
]
coordinates {
(1, 0.9179687798023224) 
(2, 0.9687500298023224) 
(3, 0.9833984524011612) 
(4, 0.9798177182674408) 
(5, 0.9879557490348816) 
(6, 0.9820963740348816) 
(7, 0.9895833432674408) 
(8, 0.9944661557674408) 
(9, 0.9537760615348816) 
(10, 0.9189453274011612) 
(11, 0.904622420668602) 
(12, 0.934895858168602) 
};


\addplot+[
]
coordinates {
(1, 0.8398437798023224) 
(2, 0.8834635615348816) 
(3, 0.8639323115348816) 
(4, 0.864583358168602) 
(5, 0.9003906548023224) 
(6, 0.7675781548023224) 
(7, 0.7929687649011612) 
(8, 0.8037109524011612) 
(9, 0.5784505382180214) 
(10, 0.8652343899011612) 
(11, 0.966145858168602) 
(12, 0.752278670668602) 
};


    
    
    \end{axis}
    \end{tikzpicture}
         \caption{Biases}
    \end{subfigure}%
    \caption{Sparsity levels found across layers by our masking algorithm, when applied to the $\text{ViT-B/32}_{\text{CLIP}}$ model. Displayed is the average across the datasets CIFAR-10, CIFAR-100, SUN397 and DTD.}
    \label{fig:vit}
\end{figure*}

\subsection{Comparison}
\label{apx:comp}

Figure \ref{fig:daaa} shows the sparsities found at different layers of the network by our self-supervised masking method in a comparison between self-supervised and supervised masking.  The corresponding overall sparsities can be seen in Table 2
. Notably, earlier ResNet blocks appear to get masked more, but within a block, the pattern reverses, with later layers seeming to get masked more. Also, the last downsampling layer is masked heavily. This appears to be generally true for both the supervised and self-supervised settings. In addition, these effects seem to remain across datasets, as can be seen from the relatively small standard deviations.

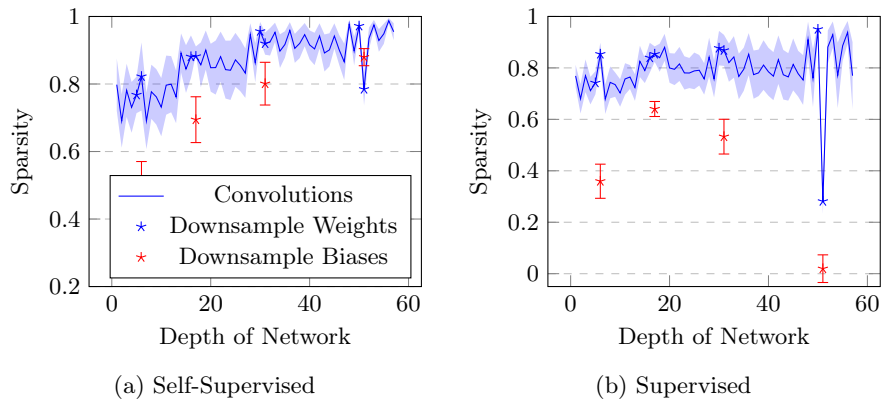
\begin{figure*}
    \centering
    \begin{subfigure}{0.5\textwidth}
        \centering
        \begin{tikzpicture}
\begin{axis}[
    width=6cm,
    xlabel={Depth of Network},
    ylabel={Sparsity},
    ymin=0.2, ymax=1,
    ytick={0.2,0.4,0.6,0.8,1},
    legend pos=south east,
    ymajorgrids=true,
    grid style=dashed,
    legend style={font=\footnotesize},
]
    
    \addplot[
        name path global=middle,
        color=blue,
    ]
    coordinates {
        (1,0.7980442047119141) 
        (2,0.69036865234375) 
        (3,0.7813178300857544) 
        (4,0.730865478515625) 
        (5,0.7676849365234375) 
        (6,0.822265625) 
        (7,0.690460205078125) 
        (8,0.776943638920784) 
        (9,0.7618865966796875) 
        (10,0.7314453125) 
        (11,0.796630859375) 
        (12,0.7994537353515625) 
        (13,0.7616500854492188) 
        (14,0.8919779509305954) 
        (15,0.8443527221679688) 
        (16,0.8807735443115234) 
        (17,0.8818359375) 
        (18,0.8567733764648438) 
        (19,0.900536447763443) 
        (20,0.8482437133789062) 
        (21,0.8486251831054688) 
        (22,0.8809356838464737) 
        (23,0.8426094055175781) 
        (24,0.8407325744628906) 
        (25,0.8715938925743103) 
        (26,0.8546791076660156) 
        (27,0.8321208953857422) 
        (28,0.946408599615097) 
        (29,0.8632287979125977) 
        (30,0.9561572074890137) 
        (31,0.919677734375) 
        (32,0.9249753952026367) 
        (33,0.9530703276395798) 
        (34,0.8973512649536133) 
        (35,0.9170780181884766) 
        (36,0.9586927741765976) 
        (37,0.9055862426757812) 
        (38,0.9204273223876953) 
        (39,0.9461788684129715) 
        (40,0.915740966796875) 
        (41,0.9038944244384766) 
        (42,0.9423476308584213) 
        (43,0.8919134140014648) 
        (44,0.9010648727416992) 
        (45,0.9421530961990356) 
        (46,0.9003000259399414) 
        (47,0.8657007217407227) 
        (48,0.9775029420852661) 
        (49,0.8975632190704346) 
        (50,0.9717733860015869) 
        (51,0.7845458984375) 
        (52,0.9349155426025391) 
        (53,0.9753208309412003) 
        (54,0.9289469718933105) 
        (55,0.9486904144287109) 
        (56,0.9873516410589218) 
        (57,0.9547562599182129) 
    };

    \addlegendentry{Convolutions}

    \addplot[name path global=lower,draw=none,forget plot] coordinates {
        (1,0.7205912330155557) 
        (2,0.6113026310268719) 
        (3,0.7024797701076558) 
        (4,0.6555560037322238) 
        (5,0.7136709375310336) 
        (6,0.7212869847288237) 
        (7,0.6070380007998086) 
        (8,0.6966441477892905) 
        (9,0.6797046753674447) 
        (10,0.6450509058254736) 
        (11,0.7090754046870902) 
        (12,0.7162056151906067) 
        (13,0.6838871992133646) 
        (14,0.8342072965590984) 
        (15,0.7736227179645137) 
        (16,0.8175931263494021) 
        (17,0.8868038569885004) 
        (18,0.7902561840124614) 
        (19,0.8425503207919416) 
        (20,0.7772473385611776) 
        (21,0.7685815413654216) 
        (22,0.8067173378301086) 
        (23,0.756824402909274) 
        (24,0.7527866918865205) 
        (25,0.7917998177482671) 
        (26,0.7783281010389498) 
        (27,0.7588873778199642) 
        (28,0.912936745477352) 
        (29,0.8051548171879324) 
        (30,0.9286190520772982) 
        (31,0.8868038569885004) 
        (32,0.8842586882117154) 
        (33,0.9244657413404782) 
        (34,0.8474605729350797) 
        (35,0.8778691662744459) 
        (36,0.9341438138434894) 
        (37,0.8626426285501574) 
        (38,0.8768441805438448) 
        (39,0.9116960480875453) 
        (40,0.8727810713695946) 
        (41,0.8559671768714768) 
        (42,0.9048934325145563) 
        (43,0.837027969024243) 
        (44,0.8518559262684339) 
        (45,0.9045566171382924) 
        (46,0.8558744898330631) 
        (47,0.8117511477768975) 
        (48,0.9642811116440507) 
        (49,0.8545126615448014) 
        (50,0.9568780045064268) 
        (51,0.736420616299354) 
        (52,0.9026567214123057) 
        (53,0.9608689501669224) 
        (54,0.8955023670845669) 
        (55,0.9186994739773048) 
        (56,0.97932020205292) 
        (57,0.9296585673754681) 
    };

    \addplot[name path global=upper,draw=none,forget plot] coordinates {
        (1,0.8754971764082724) 
        (2,0.7694346736606281) 
        (3,0.860155890063853) 
        (4,0.8061749532990262) 
        (5,0.8216989355158414) 
        (6,0.9232442652711763) 
        (7,0.7738824093564414) 
        (8,0.8572431300522775) 
        (9,0.8440685179919303) 
        (10,0.8178397191745264) 
        (11,0.8841863140629098) 
        (12,0.8827018555125183) 
        (13,0.8394129716850729) 
        (14,0.9497486053020924) 
        (15,0.9150827263714238) 
        (16,0.9439539622736448) 
        (17,0.9112789324837526) 
        (18,0.923290568917226) 
        (19,0.9585225747349444) 
        (20,0.9192400881966349) 
        (21,0.9286688248455159) 
        (22,0.9551540298628388) 
        (23,0.9283944081258823) 
        (24,0.9286784570392608) 
        (25,0.9513879674003535) 
        (26,0.9310301142930815) 
        (27,0.9053544129515202) 
        (28,0.979880453752842) 
        (29,0.921302778637263) 
        (30,0.9836953629007291) 
        (31,0.9525516117614996) 
        (32,0.965692102193558) 
        (33,0.9816749139386813) 
        (34,0.9472419569721469) 
        (35,0.9562868701025072) 
        (36,0.9832417345097058) 
        (37,0.9485298568014051) 
        (38,0.9640104642315458) 
        (39,0.9806616887383977) 
        (40,0.9587008622241554) 
        (41,0.9518216720054763) 
        (42,0.9798018292022863) 
        (43,0.9467988589786867) 
        (44,0.9502738192149646) 
        (45,0.9797495752597789) 
        (46,0.9447255620468197) 
        (47,0.9196502957045478) 
        (48,0.9907247725264815) 
        (49,0.9406137765960677) 
        (50,0.986668767496747) 
        (51,0.832671180575646) 
        (52,0.9671743637927724) 
        (53,0.9897727117154781) 
        (54,0.9623915767020542) 
        (55,0.9786813548801171) 
        (56,0.9953830800649236) 
        (57,0.9798539524609576) 
    };

    \addplot[blue,fill opacity=0.2,forget plot] fill between[of=upper and lower];

    \addplot[
        only marks,
        color=blue,
        mark=star,
    ]
    coordinates {
        (5,0.7676849365234375) 
        (6,0.822265625) 
        (16,0.8807735443115234) 
        (17,0.8818359375) 
        (30,0.9561572074890137) 
        (31,0.919677734375) 
        (50,0.9717733860015869) 
        (51,0.7845458984375) 
    };
    \addlegendentry{Downsample Weights}
    
    \addplot[
        only marks,
        color=red,
        mark=star,
          error bars/.cd,
          y dir=both,
          y explicit
    ]
    coordinates {
        (6,0.4638671875)  +- (0.10649009952285103,0.10649009952285103) 
        (17,0.6943359375)  +- (0.06786231114488941,0.06786231114488941) 
        (31,0.801025390625)  +- (0.06328236143289985,0.06328236143289985) 
        (51,0.8795166015625)  +- (0.025328045533683317,0.025328045533683317) 
    };

    \addlegendentry{Downsample Biases}
    
    
    
    \end{axis}
    \end{tikzpicture}
         \caption{Self-Supervised}
    \end{subfigure}%
    \begin{subfigure}{0.5\textwidth}
        \centering
        \begin{tikzpicture}
\begin{axis}[
    width=6cm,
    xlabel={Depth of Network},
    ylabel={Sparsity},
    ymin=-0.05, ymax=1,
    ytick={0.0, 0.2,0.4,0.6,0.8,1},
    legend pos=south east,
    ymajorgrids=true,
    grid style=dashed,
]

    \addplot[
        name path global=middle,
        color=blue,
    ]
    coordinates {
        (1,0.7688403278589249) 
        (2,0.6798095703125) 
        (3,0.769558385014534) 
        (4,0.711578369140625) 
        (5,0.7419281005859375) 
        (6,0.853515625) 
        (7,0.67864990234375) 
        (8,0.7436998337507248) 
        (9,0.735015869140625) 
        (10,0.7025909423828125) 
        (11,0.757080078125) 
        (12,0.76458740234375) 
        (13,0.7237701416015625) 
        (14,0.8254733681678772) 
        (15,0.7897300720214844) 
        (16,0.8396415710449219) 
        (17,0.853515625) 
        (18,0.839630126953125) 
        (19,0.880099818110466) 
        (20,0.7998161315917969) 
        (21,0.7944679260253906) 
        (22,0.8153991848230362) 
        (23,0.780517578125) 
        (24,0.7800979614257812) 
        (25,0.7883351594209671) 
        (26,0.7908401489257812) 
        (27,0.7571792602539062) 
        (28,0.8419736176729202) 
        (29,0.7568550109863281) 
        (30,0.8767924308776855) 
        (31,0.8681640625) 
        (32,0.8205585479736328) 
        (33,0.8427124172449112) 
        (34,0.7667655944824219) 
        (35,0.7946643829345703) 
        (36,0.8532558232545853) 
        (37,0.7738418579101562) 
        (38,0.7807865142822266) 
        (39,0.8110381215810776) 
        (40,0.784785270690918) 
        (41,0.7678442001342773) 
        (42,0.8319613039493561) 
        (43,0.7734136581420898) 
        (44,0.7649831771850586) 
        (45,0.813594400882721) 
        (46,0.7774887084960938) 
        (47,0.7529258728027344) 
        (48,0.9103271812200546) 
        (49,0.7599213123321533) 
        (50,0.9501379728317261) 
        (51,0.2813720703125) 
        (52,0.8811521530151367) 
        (53,0.9291502684354782) 
        (54,0.7709205150604248) 
        (55,0.8860363960266113) 
        (56,0.9382754564285278) 
        (57,0.7697410583496094) 
    };

    \addplot[name path global=lower,draw=none] coordinates {
        (1,0.7256475809659295) 
        (2,0.6225125430909892) 
        (3,0.7062080912122571) 
        (4,0.6579872687941944) 
        (5,0.6967541186166969) 
        (6,0.8091921651644132) 
        (7,0.6209414606692141) 
        (8,0.6818376833784636) 
        (9,0.6876283393088414) 
        (10,0.6489758327670822) 
        (11,0.6958240909924399) 
        (12,0.7181056744664055) 
        (13,0.6717575402035689) 
        (14,0.7645204040850058) 
        (15,0.7410077473858568) 
        (16,0.7816543666934215) 
        (17,0.823415947251128) 
        (18,0.7850187535593615) 
        (19,0.8277788775986608) 
        (20,0.7422083256416705) 
        (21,0.7330908763049043) 
        (22,0.7492222995765832) 
        (23,0.7157878660195374) 
        (24,0.721478212184168) 
        (25,0.7170274713215499) 
        (26,0.7313280766483224) 
        (27,0.6959981681748616) 
        (28,0.766441858035525) 
        (29,0.6851298298642704) 
        (30,0.8096991282325473) 
        (31,0.8076999368390483) 
        (32,0.7461521575804176) 
        (33,0.7618478367615528) 
        (34,0.6779092251061054) 
        (35,0.7107444712948913) 
        (36,0.7711463123398197) 
        (37,0.6801106706122746) 
        (38,0.6903746087948184) 
        (39,0.7145492164635697) 
        (40,0.6914115796151332) 
        (41,0.6743906693176511) 
        (42,0.7377004423260505) 
        (43,0.6770365194618081) 
        (44,0.6726843669083415) 
        (45,0.712568274454105) 
        (46,0.6884503495936005) 
        (47,0.6561470258503791) 
        (48,0.8409920578398866) 
        (49,0.6473631240825188) 
        (50,0.9270950650506864) 
        (51,0.23173116049219372) 
        (52,0.8164424040677453) 
        (53,0.8779387376364861) 
        (54,0.6433082063595339) 
        (55,0.826816803889687) 
        (56,0.894258076433624) 
        (57,0.6393174284566927) 
    };

    \addplot[name path global=upper,draw=none] coordinates {
        (1,0.8120330747519202) 
        (2,0.7371065975340108) 
        (3,0.8329086788168109) 
        (4,0.7651694694870556) 
        (5,0.7871020825551781) 
        (6,0.8978390848355868) 
        (7,0.7363583440182859) 
        (8,0.805561984122986) 
        (9,0.7824033989724086) 
        (10,0.7562060519985428) 
        (11,0.8183360652575601) 
        (12,0.8110691302210945) 
        (13,0.7757827429995561) 
        (14,0.8864263322507486) 
        (15,0.8384523966571119) 
        (16,0.8976287753964223) 
        (17,0.883615302748872) 
        (18,0.8942415003468885) 
        (19,0.9324207586222713) 
        (20,0.8574239375419233) 
        (21,0.8558449757458769) 
        (22,0.8815760700694892) 
        (23,0.8452472902304626) 
        (24,0.8387177106673945) 
        (25,0.8596428475203843) 
        (26,0.85035222120324) 
        (27,0.8183603523329509) 
        (28,0.9175053773103155) 
        (29,0.8285801921083858) 
        (30,0.9438857335228238) 
        (31,0.9286281881609517) 
        (32,0.894964938366848) 
        (33,0.9235769977282696) 
        (34,0.8556219638587383) 
        (35,0.8785842945742494) 
        (36,0.9353653341693509) 
        (37,0.8675730452080379) 
        (38,0.8711984197696347) 
        (39,0.9075270266985854) 
        (40,0.8781589617667027) 
        (41,0.8612977309509036) 
        (42,0.9262221655726617) 
        (43,0.8697907968223716) 
        (44,0.8572819874617756) 
        (45,0.9146205273113369) 
        (46,0.866527067398587) 
        (47,0.8497047197550897) 
        (48,0.9796623046002226) 
        (49,0.8724795005817878) 
        (50,0.9731808806127658) 
        (51,0.3310129801328063) 
        (52,0.9458619019625282) 
        (53,0.9803617992344703) 
        (54,0.8985328237613157) 
        (55,0.9452559881635356) 
        (56,0.9822928364234317) 
        (57,0.900164688242526) 
    };

    \addplot[blue,fill opacity=0.2] fill between[of=upper and lower];

    \addplot[
        only marks,
        color=blue,
        mark=star,
    ]
    coordinates {
        (5,0.7419281005859375) 
        (6,0.853515625) 
        (16,0.8396415710449219) 
        (17,0.853515625) 
        (30,0.8767924308776855) 
        (31,0.8681640625) 
        (50,0.9501379728317261) 
        (51,0.2813720703125) 
    };

    \addplot[
        only marks,
        color=red,
        mark=star,
          error bars/.cd,
          y dir=both,
          y explicit
    ]
    coordinates {
        (6,0.3594324697208475)  +- (0.06634878027915246,0.06634878027915246) 
        (17,0.6402549586556081)  +- (0.02917863509439192,0.02917863509439192) 
        (31,0.532773921619291)  +- (0.06756787525570894,0.06756787525570894) 
        (51,0.019485770320276347)  +- (0.05375641717972365,0.05375641717972365) 
    };



    \end{axis}
    \end{tikzpicture}
         \caption{Supervised}
    \end{subfigure}
    \caption{Sparsity levels found across layers by our Self-Masking Networks on a SwAV-pretrained
ResNet-50, compared to our supervised masking algorithm. Average across the datasets CIFAR-100, CIFAR-10, SUN397, and DTD. Standard deviations included.}
    \label{fig:daaa}
\end{figure*}

\section{Progressive Sparsity}

\begin{table*}
    \caption{\textbf{Updating Ramanujan~\etal~'s method \cite{ramanujanWhatHiddenRandomly2020}, such that it decreases the number of active weights as training progresses. } The progression is a linear progression with respect to the percentage of active weights, starting at 100\% and going down to the 'found' sparsity.
      All numbers are reported by masking with a self-supervised loss, starting from a ResNet-50 pretrained with SwAV.
      \\}
    \label{tab:progressive}
    
\centering
\small
  \setlength{\tabcolsep}{3pt}
     \begin{tabular}{@{}l c cccccccc}
     \toprule
     Masking mechanism & Sparsity & \textsc{dtd} & \textsc{eurosat} & \textsc{flowers} & \textsc{ucf101} \\
     \midrule
\multicolumn{4}{l}{\textit{$k$-NN evaluation}} \\
        Ramanujan~\etal~\cite{ramanujanWhatHiddenRandomly2020} & progressive &	0.658 &	0.957 &	0.895 & 0.549 \\
        Ramanujan~\etal~\cite{ramanujanWhatHiddenRandomly2020} & found  &	0.671 &	\textbf{0.973} &	0.903	 &	 \textbf{0.572} \\
        Threshold (Ours)  & found  &	\textbf{0.674} & 0.971 &	\textbf{0.920}	& 0.549 \\

     \midrule
\multicolumn{4}{l}{\textit{Linear probe evaluation}} \\
          Ramanujan~\etal~\cite{ramanujanWhatHiddenRandomly2020} & progressive &	0.706 &	0.981 & 0.983 & 0.681 \\
         Ramanujan~\etal~\cite{ramanujanWhatHiddenRandomly2020} & found & \textbf{0.733} & \textbf{0.984} & \textbf{0.985} & 0.690 \\
        Threshold (Ours)& found & 0.714 & 0.983 & 0.983 &\textbf{0.697} \\
        \midrule 
        \midrule
                \multicolumn{2}{l}{Sparsity level found by our method:} & 98.8\% & 96.7\% & 98.6\% & 95.7\% \\
    \bottomrule
    \end{tabular}
\end{table*} 
Table \ref{tab:progressive} shows an additional experiment, where we progressively deactivate more weights during training, using a fixed schedule, allowing Ramanujan~\etal's method to start with every weight activated. Unfortunately, this only seems to degrade its performance.

\pagebreak
\section{Selective masking}

This section evaluates whether the storage cost can be further reduced by selecting a smaller subset of the weights to be masked, and leaving the others unchanged. Following the results from Appendix \ref{apx:comp}, where earlier layers where masked more than subsequent layers, we see what performance we get when only masking the first layer (L1), the first and second (L1 + L2) and the first, second and third (L1 + L2 + L3) layers. The results are compared to the baseline of masking all layers (L1+L2+L3+L4). In addition, we perform separate experiments where we randomly freeze each individual mask, corresponding to a single weight, with probability $p$, such that we run 9 additional experiments where 10\%, 20\%, ... up to 90\% of masks are trained. The results are shown in Figure \ref{fig:sm}. In addition, two experiments are shown where only 0.9\% and 6.1\% of weights are masked, corresponding to the number of parameters in only the first layer or only the first and second layers.

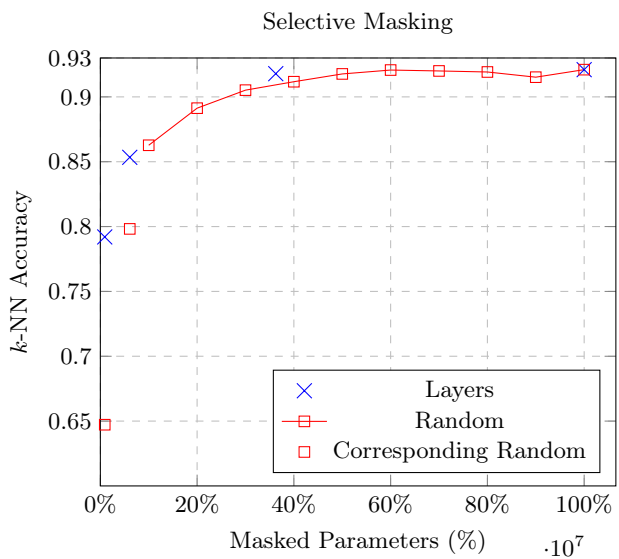
\begin{figure}
\centering
    \centering
    
    \begin{tikzpicture}
    \begin{axis}[
        title={Selective Masking},
        xlabel={Masked Parameters (\%)},
        ylabel={$k$-NN Accuracy},
        xmin=0, xmax=25000000,
        ymin=0.60, ymax=0.93,
        xtick={0,4692518.4,9385036.8,14077555.2,18770073.6,23462592},
        xticklabels={0\%,20\%,40\%,60\%,80\%,100\%},
        ytick={0.65,0.70,0.75,0.8,0.85,0.9,0.93},
        legend pos=south east,
        ymajorgrids=true,
        xmajorgrids=true,
        grid style=dashed,
    ]

    \addplot[
        color=blue,
        mark=x,
        mark size=4pt,
        only marks
        ]
        coordinates {
        (213504,0.792)(1426944,0.8534)(8506880,0.9179)(23462592,0.921)
        };
        \addlegendentry{Layers}

    \addplot[
        color=red,
        mark=square,
        ]
        coordinates {
        (23462592,0.921)(21116332.8,0.9152)(18770073.6,0.9192)(16423814.4,0.92)(14077555.2,0.9207)(11731296.0,0.9177)(9385036.8,0.9117)(7038777.6,0.9052)(4692518.4,0.8913)(2346259.2,0.8627)
        };
        \addlegendentry{Random}
    
    \addplot[
        color=red,
        mark=square,
        mark size=2pt,
        only marks
        ]
        coordinates {
        (213504,0.6472)(1426944,0.7982)(8506880,0.567)
        };
        \addlegendentry{Corresponding Random}
    
    \end{axis}
    \end{tikzpicture}
        \caption{Accuracy when only masking the first N ResNet-50 layers (L1, L2, L3, L4) with corresponding experiments where random weights are masked instead, and accuracy when training 10\%, 20\%, ... up to 90\% of masks randomly across all weights in each weight matrix. The resulting total number of trainable masks is given on the x-axis. These experiments where run with a ResNet-50 from SWaV.}
        \label{fig:sm}
\end{figure}

It can be seen that, for this dataset, the model is able to maintain its accuracy with little loss in accuracy when training a mask on only half of the parameters. The graph also shows that, on this experimental setting, masking the first $N$ layers yields higher accuracies than randomly masking the same number of weights across the whole network. It is thus possible to further reduce storage costs of the masks with little loss of accuracy by at least 50\% using this technique. Corresponding to an up to 64x reduction of storage costs of the fine-tuned model when compared to full fine-tuning.

Additionally, we run the same experiment on Cifar-100, but here we compare the random selection of trainable masks across layers with a similar method where we select the top-k\% masks corresponding to weights with the highest magnitude. The results are displayed in Figure \ref{fig:c100maxmag}, where it can be seen that the maximum magnitude approach provides some benefit when only a tiny proportion of weights are trained, however, by at this level, the accuracy is already far below the baseline when masking all weights. Surprisingly though, a small gain in accuracy appears to be made when only masking around half the weights as opposed to all.

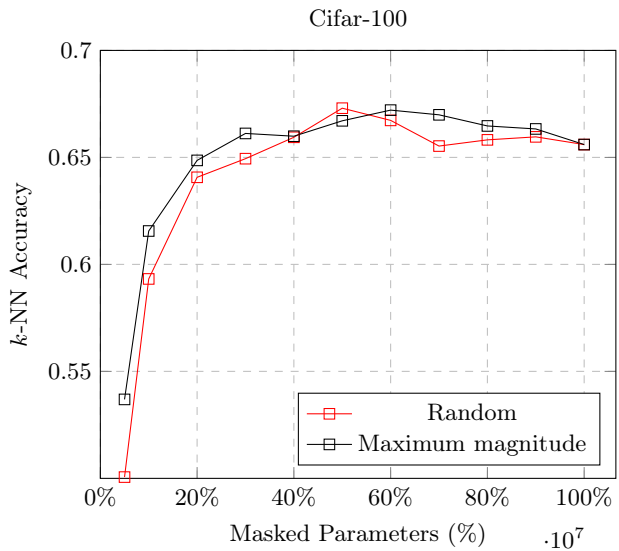
\begin{figure}
        \centering
    
    \begin{tikzpicture}
    \begin{axis}[
        title={Cifar-100},
        xlabel={Masked Parameters (\%)},
        ylabel={$k$-NN Accuracy},
        xmin=0, xmax=25000000,
        ymin=0.50, ymax=0.70,
        xtick={0,4692518.4,9385036.8,14077555.2,18770073.6,23462592},
        xticklabels={0\%,20\%,40\%,60\%,80\%,100\%},
        ytick={0.55,0.60,0.65,0.70},
        legend pos=south east,
        ymajorgrids=true,
        xmajorgrids=true,
        grid style=dashed,
    ]
     

    \addplot[
        color=red,
        mark=square,
        ]
        coordinates {
        (23462592,0.656)(21116332.8,0.6596)(18770073.6,0.6582)(16423814.4,0.6553)(14077555.2,0.6673)(11731296.0,0.673)(9385036.8,0.6594)(7038777.6,0.6494)(4692518.4,0.6407)(2346259.2,0.5932)(1173129.5,0.5006)
        };
        \addlegendentry{Random}

    \addplot[
        color=black,
        mark=square,
        ]
        coordinates {
        (23462592,0.656)(21116332.8,0.6633)(18770073.6,0.6647)(16423814.4,0.6699)(14077555.2,0.6721)(11731296.0,0.6671)(9385036.8,0.6599)(7038777.6,0.6612)(4692518.4,0.6486)(2346260.2,0.6156)
        (1173129.5,0.5369)
        };
        \addlegendentry{Maximum magnitude}
    
    
    \end{axis}
    \end{tikzpicture}
    \caption{Accuracy when training 5\%, 10\%, 20\%, ... up to 90\% of masks randomly across all weights in each weight matrix, compared to training 5\%, 10\%, 20\%, ... up to 90\% of the masks corresponding to the weights with the highest magnitudes.}
    \label{fig:c100maxmag}
\end{figure}%

\section{Dispatcher Clusters}

\begin{figure*}
     \centering
     \begin{subfigure}[b]{0.25\textwidth}
        \includegraphics[width=\columnwidth]{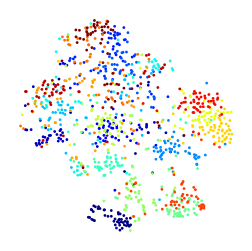}
        \caption{\textbf{Original embeddings.}}
    \end{subfigure}
    \hfill
    \begin{subfigure}[b]{0.25\textwidth}
        \includegraphics[width=\columnwidth]{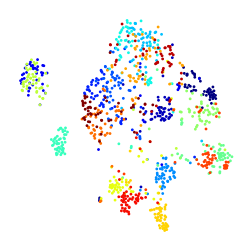}
        \caption{\textbf{Adapted embeddings.}}
    \end{subfigure}
    \hfill
     \begin{subfigure}[b]{0.25\textwidth}
        \includegraphics[width=\columnwidth]{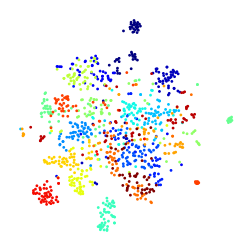}
        \caption{\textbf{Cascade embeddings.}}
    \end{subfigure}
    \caption{We compare the embeddings from (a) the original pretrained model, (b) the adapted model and (c) the cascade model on first 20 classes of \textsc{CIFAR100}. Ground-truth classes coded by color. \label{fig:tsne}}
\end{figure*}
In~\Cref{fig:tsne} we compare the final embeddings of the original pretrained model, the domain adapted and the cascade model on a subset of \textsc{CIFAR100}. 
We find that our self-mask adapted embeddings create a few more distinct clusters, but a massive difference is not visible when compared to the cascade embeddings. Both however are a clear improvement over the unadapted embeddings from the pre-trained model.

We also show in Figure \ref{fig:5l}, for each cluster of the Cascade model, the cumulative distribution of datapoints belonging to the most common classes in each cluster. This plot shows that the dispatcher is effective at creating dataset splits that have more homogenous distributions. For example, it can be read from the graph that, for each cluster, the top 20 classes represent at least 60\% of the datapoints in that cluster, and three of these clusters are even more homogenous, with the top 20 classes representing at least 80\% of the datapoints.

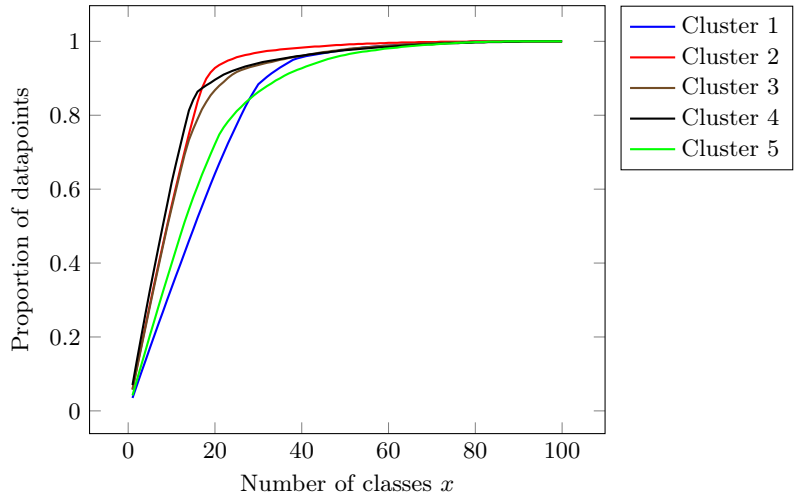
\begin{figure*}
    \centering
    \begin{tikzpicture}
    \begin{axis}[xlabel=Number of classes $x$, ylabel=Proportion of datapoints, legend pos=outer north east, no markers]
    
    \addplot+[thick] coordinates {
    (1,0.035) (2,0.069) (3,0.103) (4,0.136) (5,0.170) (6,0.203) (7,0.236) (8,0.268) (9,0.301) (10,0.333) (11,0.365) (12,0.396) (13,0.428) (14,0.460) (15,0.491) (16,0.523) (17,0.553) (18,0.583) (19,0.613) (20,0.642) (21,0.670) (22,0.697) (23,0.724) (24,0.749) (25,0.774) (26,0.799) (27,0.823) (28,0.845) (29,0.865) (30,0.884) (31,0.894) (32,0.904) (33,0.913) (34,0.921) (35,0.929) (36,0.936) (37,0.943) (38,0.950) (39,0.954) (40,0.957) (41,0.960) (42,0.962) (43,0.964) (44,0.966) (45,0.968) (46,0.970) (47,0.972) (48,0.974) (49,0.975) (50,0.977) (51,0.978) (52,0.980) (53,0.981) (54,0.983) (55,0.984) (56,0.985) (57,0.986) (58,0.987) (59,0.988) (60,0.988) (61,0.989) (62,0.990) (63,0.991) (64,0.991) (65,0.992) (66,0.993) (67,0.993) (68,0.994) (69,0.994) (70,0.995) (71,0.995) (72,0.996) (73,0.996) (74,0.997) (75,0.997) (76,0.997) (77,0.998) (78,0.998) (79,0.999) (80,0.999) (81,0.999) (82,0.999) (83,0.999) (84,1.000) (85,1.000) (86,1.000) (87,1.000) (88,1.000) (89,1.000) (90,1.000) (91,1.000) (92,1.000) (93,1.000) (94,1.000) (95,1.000) (96,1.000) (97,1.000) (98,1.000) (99,1.000) (100,1.000)
    };
    
    \addlegendentry{Cluster 1}
    
    \addplot+[thick] coordinates {
    (1,0.059) (2,0.117) (3,0.174) (4,0.231) (5,0.287) (6,0.343) (7,0.397) (8,0.450) (9,0.503) (10,0.555) (11,0.605) (12,0.654) (13,0.702) (14,0.749) (15,0.793) (16,0.837) (17,0.873) (18,0.899) (19,0.916) (20,0.928) (21,0.936) (22,0.942) (23,0.948) (24,0.952) (25,0.957) (26,0.960) (27,0.963) (28,0.965) (29,0.968) (30,0.970) (31,0.972) (32,0.974) (33,0.975) (34,0.976) (35,0.978) (36,0.979) (37,0.980) (38,0.981) (39,0.982) (40,0.983) (41,0.984) (42,0.985) (43,0.986) (44,0.986) (45,0.987) (46,0.988) (47,0.989) (48,0.990) (49,0.990) (50,0.991) (51,0.992) (52,0.992) (53,0.993) (54,0.993) (55,0.994) (56,0.994) (57,0.994) (58,0.995) (59,0.995) (60,0.996) (61,0.996) (62,0.996) (63,0.997) (64,0.997) (65,0.997) (66,0.997) (67,0.998) (68,0.998) (69,0.998) (70,0.998) (71,0.998) (72,0.999) (73,0.999) (74,0.999) (75,0.999) (76,0.999) (77,0.999) (78,0.999) (79,1.000) (80,1.000) (81,1.000) (82,1.000) (83,1.000) (84,1.000) (85,1.000) (86,1.000) (87,1.000) (88,1.000) (89,1.000) (90,1.000) (91,1.000) (92,1.000) (93,1.000) (94,1.000) (95,1.000) (96,1.000) (97,1.000) (98,1.000) (99,1.000) (100,1.000)
    };
    
    \addlegendentry{Cluster 2}

    \addplot+[thick] coordinates {
    (1,0.057) (2,0.114) (3,0.170) (4,0.226) (5,0.282) (6,0.337) (7,0.391) (8,0.445) (9,0.498) (10,0.548) (11,0.597) (12,0.646) (13,0.693) (14,0.733) (15,0.763) (16,0.790) (17,0.816) (18,0.836) (19,0.854) (20,0.868) (21,0.880) (22,0.890) (23,0.901) (24,0.910) (25,0.917) (26,0.922) (27,0.926) (28,0.930) (29,0.933) (30,0.936) (31,0.939) (32,0.942) (33,0.944) (34,0.947) (35,0.950) (36,0.952) (37,0.955) (38,0.957) (39,0.959) (40,0.961) (41,0.963) (42,0.965) (43,0.967) (44,0.969) (45,0.971) (46,0.972) (47,0.974) (48,0.975) (49,0.977) (50,0.978) (51,0.980) (52,0.981) (53,0.982) (54,0.983) (55,0.984) (56,0.985) (57,0.986) (58,0.987) (59,0.988) (60,0.989) (61,0.989) (62,0.990) (63,0.991) (64,0.992) (65,0.992) (66,0.993) (67,0.993) (68,0.994) (69,0.994) (70,0.995) (71,0.995) (72,0.996) (73,0.996) (74,0.997) (75,0.997) (76,0.997) (77,0.998) (78,0.998) (79,0.998) (80,0.999) (81,0.999) (82,0.999) (83,0.999) (84,0.999) (85,0.999) (86,0.999) (87,1.000) (88,1.000) (89,1.000) (90,1.000) (91,1.000) (92,1.000) (93,1.000) (94,1.000) (95,1.000) (96,1.000) (97,1.000) (98,1.000) (99,1.000) (100,1.000)
    };
    
    \addlegendentry{Cluster 3}

    \addplot+[thick] coordinates {
    (1,0.070) (2,0.136) (3,0.200) (4,0.264) (5,0.326) (6,0.385) (7,0.444) (8,0.502) (9,0.559) (10,0.617) (11,0.668) (12,0.717) (13,0.766) (14,0.813) (15,0.842) (16,0.864) (17,0.873) (18,0.881) (19,0.888) (20,0.896) (21,0.903) (22,0.910) (23,0.915) (24,0.920) (25,0.924) (26,0.928) (27,0.932) (28,0.935) (29,0.938) (30,0.941) (31,0.944) (32,0.946) (33,0.948) (34,0.950) (35,0.952) (36,0.954) (37,0.956) (38,0.958) (39,0.960) (40,0.961) (41,0.963) (42,0.965) (43,0.966) (44,0.968) (45,0.969) (46,0.970) (47,0.972) (48,0.973) (49,0.975) (50,0.976) (51,0.977) (52,0.978) (53,0.979) (54,0.980) (55,0.981) (56,0.982) (57,0.983) (58,0.984) (59,0.984) (60,0.985) (61,0.986) (62,0.987) (63,0.988) (64,0.989) (65,0.989) (66,0.990) (67,0.991) (68,0.991) (69,0.992) (70,0.992) (71,0.993) (72,0.994) (73,0.994) (74,0.995) (75,0.995) (76,0.995) (77,0.996) (78,0.996) (79,0.996) (80,0.997) (81,0.997) (82,0.997) (83,0.998) (84,0.998) (85,0.998) (86,0.998) (87,0.999) (88,0.999) (89,0.999) (90,0.999) (91,1.000) (92,1.000) (93,1.000) (94,1.000) (95,1.000) (96,1.000) (97,1.000) (98,1.000) (99,1.000) (100,1.000)
    };
    
    \addlegendentry{Cluster 4}

    \addplot+[thick, color=green] coordinates {
    (1,0.042) (2,0.083) (3,0.123) (4,0.163) (5,0.204) (6,0.243) (7,0.283) (8,0.322) (9,0.360) (10,0.398) (11,0.436) (12,0.474) (13,0.511) (14,0.545) (15,0.578) (16,0.608) (17,0.639) (18,0.667) (19,0.695) (20,0.722) (21,0.748) (22,0.767) (23,0.783) (24,0.797) (25,0.811) (26,0.822) (27,0.834) (28,0.844) (29,0.854) (30,0.863) (31,0.871) (32,0.879) (33,0.886) (34,0.894) (35,0.900) (36,0.907) (37,0.913) (38,0.918) (39,0.923) (40,0.927) (41,0.932) (42,0.936) (43,0.940) (44,0.944) (45,0.948) (46,0.952) (47,0.955) (48,0.958) (49,0.961) (50,0.963) (51,0.966) (52,0.968) (53,0.970) (54,0.972) (55,0.973) (56,0.975) (57,0.976) (58,0.978) (59,0.980) (60,0.981) (61,0.982) (62,0.983) (63,0.985) (64,0.986) (65,0.987) (66,0.988) (67,0.989) (68,0.990) (69,0.991) (70,0.991) (71,0.992) (72,0.993) (73,0.994) (74,0.994) (75,0.995) (76,0.996) (77,0.996) (78,0.997) (79,0.997) (80,0.998) (81,0.998) (82,0.998) (83,0.999) (84,0.999) (85,0.999) (86,0.999) (87,1.000) (88,1.000) (89,1.000) (90,1.000) (91,1.000) (92,1.000) (93,1.000) (94,1.000) (95,1.000) (96,1.000) (97,1.000) (98,1.000) (99,1.000) (100,1.000)
    };
    
    \addlegendentry{Cluster 5}

    
    \end{axis}
    \end{tikzpicture}

    \caption{Proportion of datapoints made up of the top $x$ most common classes for each of the five cascade clusters (training set).}
    \label{fig:5l}
\end{figure*}

\section{Error Bars}

Due to computational and time constraints, we were only able to run every experiment once. However, to give an idea of the level of noise in our results, we show the variance of our thresholding method by running it 5x on two different datasets. The results can be seen in Table \ref{tab:variance}.

\begin{table*}
    \caption{\textbf{Variance in $k$-NN accuracies with the Threshold Masking Mechanism.} This table shows the results of five separate runs of the same experiment on the "dtd" and "ucf101" datasets, including the mean across experiments and the standard deviation of the results.}
    \label{tab:variance}
    
\centering
\small
  \setlength{\tabcolsep}{3pt}
     \begin{tabular}{@{}l c cccccccc}
     \toprule
     Dataset & Run & R1 & R2 & R3 & R4 & R5 & Mean & STD \\
     \midrule
\multicolumn{2}{l}{\textit{$k$-NN evaluation}} \\
        \textsc{dtd}  & Threshold (Ours) & 0.685 & 0.678 & 0.670 & 0.678 & 0.671 & 0.676 & 0.005 \\
        \textsc{ucf101} & Threshold (Ours) & 0.569 &0.576 & 0.557 & 0.574 & 0.568 & 0.569 & 0.007 \\
    \bottomrule
    \end{tabular}
\end{table*}

\section{Datasets}
\label{app:ds}
 We use datasets from \cite{Loedeman2022prompt}, namely \textsc{CIFAR100} \& \textsc{CIFAR10} \cite{krizhevsky2009learning}, \textsc{Oxford Flowers}~\cite{nilsback2008automated}, \textsc{Food101}~\cite{bossard2014food}, \textsc{EuroSAT}~\cite{helber2019eurosat}, \textsc{SUN397}~\cite{xiao2010sun}, \textsc{UCF101}~\cite{soomro2012ucf101}, \textsc{Oxford-IIIT Pets}~\cite{parkhi2012cats}, \textsc{DTD}~\cite{cimpoi2014sammy}. For the Model Cascade, we also use a fine-grained subset of iNaturalist~\cite{van2018inaturalist}, \textsc{iNatLoc500} from \cite{cole2022label}.

\section{Compute cost}

 At least 24GB of GPU memory is necessary to run the most demanding experiments. On an A100 gpu, most experiments (training a single model) take less than 24h. Self-supervised fine-tuning is usually slower than supervised fine-tuning. Masking one of the smaller datasets with the supervised algorithm only takes about 6 hours. Masking was not noticeably slower than full fine-tuning. So the costs of training one model/mask will vary between \$12 and \$48 at \$2 per GPU-hour. This applies to the ResNet50 and ViT-B models. ResNet18 is a lot cheaper/faster.

\section{Broader Impact}

Our technique could help with reducing storage and network tansfer costs. However, it is not more computationally efficient and thus the (more important) energy cost of running these masked models remains equal, or increases if running the Cascade model. Another interesting aspect of mask-learning in general through the pass-through trick is that it enables the learning of \textit{discrete} components like text rather than only being able to learn images. Hence, the pass-through trick could be used to perform text reconstruction attacks. Existing approaches to privacy-preserving model training thus should also protect against this.

\section{Source Code}

Source code is available at \url{https://github.com/alvitawa/UnsupervisedMasking}. 

\end{document}